%% file: main.tex
\documentclass{article} % For LaTeX2e
\usepackage{iclr2023_conference,times}

\input{math_commands}

\input{terminology}

\usepackage[ruled,linesnumbered,vlined]{algorithm2e}
\usepackage{hyperref}
\usepackage{url}
\usepackage{cleveref}
\usepackage{enumerate}
\usepackage[pdftex]{graphicx}
\usepackage{caption}
\usepackage{subcaption,enumitem}
\usepackage{natbib}
\usepackage{wrapfig}
\usepackage{array,booktabs,makecell,multirow}
\usepackage{bm}
\usepackage{tikz}
\usepackage{minitoc}
\usepackage{thm-restate}

\title{Order Matters: Agent-by-agent Policy Optimization}

\author{Xihuai Wang$^{1,2}$\thanks{Work done at Digital Brain Lab. $\quad^{\dagger}$Correspondence to Zheng Tian $<$tianzheng@shanghaitech.edu.cn$>$ and Weinan Zhang $<$wnzhang@sjtu.edu.cn$>$.},\, Zheng Tian$^{3\dagger}$,\, Ziyu Wan$^{1,2}$,\, Ying Wen$^{1}$,\, Jun Wang$^{2,4}$,\, Weinan Zhang$^{1\dagger}$ \\
$^{1}$ Shanghai Jiao Tong University,\,$^{2}$ Digital Brain Lab,\, \\
$^{3}$ ShanghaiTech University,\,$^{4}$ University College London \\
}

\iclrfinalcopy % Uncomment for camera-ready version, but NOT for submission.

\Crefname{preproposition}{Prop.}{Prop.}
\Crefname{proposition}{Prop.}{Prop.}
\Crefname{appendix}{Appx.}{Appx.}
\Crefname{thm}{Thm.}{Thm.}
\Crefname{equation}{Eq.}{Eq.}
\Crefname{section}{Sec.}{Sec.}
\Crefname{figure}{Fig.}{Fig.}
\Crefname{table}{Tab.}{Tab.}
\Crefname{algorithm}{Alg.}{Alg.}

\renewcommand{\mod}[1]{{\color{black}#1}}

\begin{document}

\setlength{\tabcolsep}{4pt}
\doparttoc % Tell to minitoc to generate a toc for the parts
\faketableofcontents % Run a fake tableofcontents command for the partocs
\maketitle
\vspace{-10pt}
\begin{abstract}
    While multi-agent trust region algorithms have achieved great success empirically in solving coordination tasks, most of them,  however, suffer from a non-stationarity problem since agents update their policies simultaneously. In contrast, a sequential scheme that updates policies agent-by-agent provides another perspective and shows strong performance. However, sample inefficiency and lack of monotonic improvement guarantees for each agent are still the two significant challenges for the sequential scheme. In this paper, we propose the \textbf{A}gent-by-\textbf{a}gent \textbf{P}olicy \textbf{O}ptimization (A2PO) algorithm to improve the sample efficiency and retain the guarantees of monotonic improvement for each agent during training. We justify the tightness of the monotonic improvement bound compared with other trust region algorithms. From the perspective of sequentially updating agents, we further consider the effect of agent updating order and extend the theory of non-stationarity into the sequential update scheme. To evaluate A2PO, we conduct a comprehensive empirical study on four benchmarks: StarCraftII, Multi-agent MuJoCo, Multi-agent Particle Environment, and Google Research Football full game scenarios. A2PO consistently outperforms strong baselines.
\end{abstract}
\vspace{-15pt}

\input{sections/1-intro}

\input{sections/2-background_related_works}
\input{sections/3-analysis}

\input{sections/4-method}
\input{sections/5-experiments}
\input{sections/6-conclusion}

\clearpage
\newpage

\textbf{Acknowledgements.} The SJTU team is partially supported by ``New Generation of AI 2030'' Major Project (2018AAA0100900), the Shanghai Municipal Science and Technology Major Project (2021SHZDZX0102), the Shanghai Sailing Program (21YF1421900), the National Natural Science Foundation of China (62076161, 62106141). Xihuai Wang and Ziyu Wan are supported by Wu Wen Jun Honorary Scholarship, AI Institute, Shanghai Jiao Tong University. We thank Yan Song and He Jiang for their help in the football experiments.

\textbf{Ethics Statement.}
Our method and algorithm do not involve any adversarial attack, and will not endanger human security.
All our experiments are performed in the simulation environment, which does not involve ethical and fair issues.

\textbf{Reproducibility Statement.}
The source code of this paper is available at \url{https://anonymous.4open.science/r/A2PO}. 
We provide proofs in \cref{app:proof}, including the proofs of intuitive sequential update, monotonic policy improvement of A2PO, incrementally tightened bound of A2PO and monotonic policy improvement of MAPPO, CoPPO and HAPPO.
We specify all the experiments implementation details, the experiments setup, and the additional results in the \cref{app:exp}.
The related works of coordinate descent are shown in \cref{app:rw}.
\bibliography{conference}
\bibliographystyle{iclr2023_conference}

\renewcommand \thepart{}
\renewcommand \partname{}
\newpage

\appendix
\addcontentsline{toc}{section}{Appendix} % Add the appendix text to the document TOC
\part{Supplementary Material} % Start the appendix part
\parttoc % Insert the appendix TOC% Add the appendix text to the document TOC
\newpage
\input{sections/appendix-proof}
\newpage
\input{sections/appendix-experiments}
\newpage
\input{sections/appendix-misc}

\end{document}

%% file: math_commands.tex
\usepackage{amsmath,amsfonts,bm}

\providecommand{\jntpi}{\bm{\pi}}
\providecommand{\jntact}{\bm{a}}
\providecommand{\jntactb}{\bar{\bm{a}}}

\providecommand{\curpi}{\pi}
\providecommand{\tarpi}{\bar{\pi}}
\providecommand{\jntcurpi}{\jntpi}
\providecommand{\jnttarpi}{\bm{\tarpi}}
\providecommand{\jntinterpi}[1]{\hat{\bm{\pi}}^{#1}}

\providecommand{\V}{V}

\providecommand{\T}{\mathcal{T}}

\providecommand{\J}[1]{\mathcal{J}(#1)}
\providecommand{\Lo}[2]{\mathcal{L}_{#1}(#2)}
\providecommand{\G}{\mathcal{G}_{\jntcurpi}(\jnttarpi)}

\newtheorem{predefinition}{Definition}
\newenvironment{definition}[1]
{
    \begin{predefinition}[\emph{#1}]
        }{
    \end{predefinition}
}

\newtheorem{theorem}{Theorem}

\newtheorem{preproposition}{Proposition}

\newenvironment{proof}{\emph{Proof. }}{\hfill$\square$}

\newtheorem{lemma}{Lemma}
\newtheorem{corollary}{Corollary}

\def\eqref#1{eq.~\ref{#1}}
\def\Eqref#1{Eq.~(\ref{#1})}
\def\1{\bm{1}}

\DeclareMathAlphabet{\mathsfit}{\encodingdefault}{\sfdefault}{m}{sl}
\SetMathAlphabet{\mathsfit}{bold}{\encodingdefault}{\sfdefault}{bx}{n}

\newcommand{\E}{\mathbb{E}}

\DeclareMathOperator*{\argmax}{arg\,max}
\DeclareMathOperator*{\argmin}{arg\,min}

%% file: terminology.tex
\providecommand{\alg}{A2PO}

\providecommand{\trace}{preceding-agent off-policy correction}
\providecommand{\seqrpi}{RPISA}
\providecommand{\seqrpippo}{RPISA-PPO}

%% file: sections/1-intro.tex
\section{Introduction}
\vspace{-5pt}
Trust region learning methods in reinforcement learning (RL)~\citep{Kakade2002} have achieved great success in solving complex tasks, from single-agent control tasks~\citep{Andrychowicz2020} to multi-agent applications~\citep{Albrecht2018,Ye2020}. The methods deliver superior and stable performances because of their theoretical guarantees of monotonic policy improvement. Recently, several works that adopt trust region learning in multi-agent reinforcement learning (MARL) have been proposed, including algorithms in which agents independently update their policies using trust region methods~\citep{Witt2020,Yu2022} and algorithms that coordinate agents' policies during the update process~\citep{Wu2021,kuba2022trust}. Most algorithms update the agents simultaneously, that is, all agents perform policy improvement at the same time and cannot observe the change of other agents, as shown in \Cref{fig:tax}c. The simultaneous update scheme brings about the non-stationarity problem, i.e., the environment dynamic changes from one agent's perspective as other agents also change their policies~\citep{HernandezLeal2017}.

\begin{wrapfigure}{r}{0.5\linewidth}
    \centering
    \vspace{-10pt}
    \includegraphics[width=\linewidth]{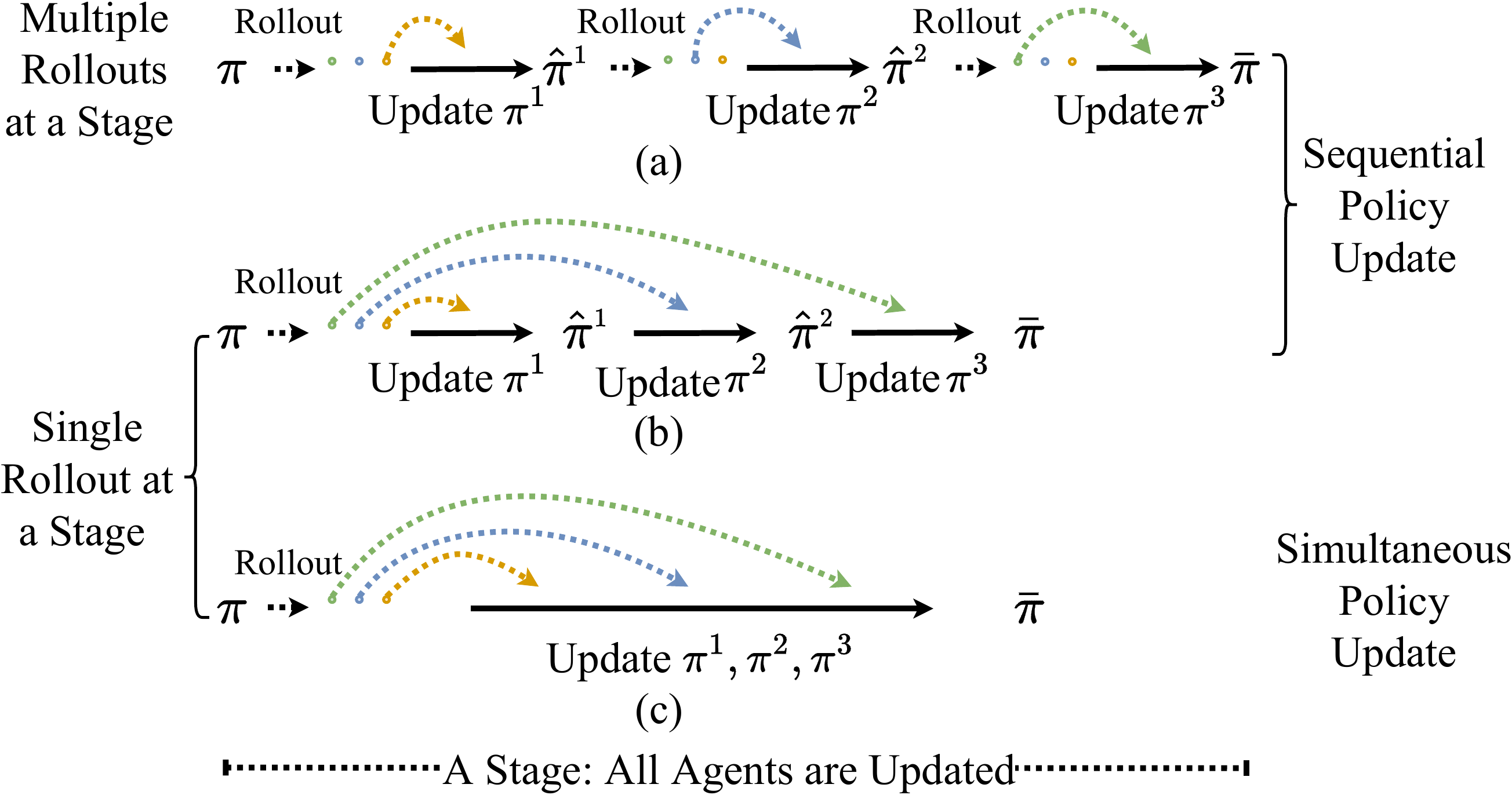}
    \caption{The taxonomy on the rollout scheme and the policy update scheme.}
    % \caption{The taxonomy.}
    % \hspace{-5pt}
    \vspace{-8pt}
    \label{fig:tax}
\end{wrapfigure}
In contrast to the simultaneous update scheme, algorithms that sequentially execute agent-by-agent updates allow agents to perceive changes made by preceding agents, presenting another perspective for analyzing inter-agent interaction \citep{Gemp2022}.
% The sequential scheme has drawn some research interest recently.
\citet{Bertsekas2021} proposed a sequential update framework, named Rollout and Policy Iteration for a Single Agent (\seqrpi{}) in this paper, which performs a rollout every time an agent updates its policy (\Cref{fig:tax}a). \seqrpi{} effectively turns non-stationary MARL problems into stationary single agent reinforcement learning (SARL) ones. It retains the theoretical properties of the chosen SARL base algorithm, such as the monotonic improvement~\citep{Kakade2002}. However, it is sample-inefficient since it only utilizes $1/n$ of the collected samples to update $n$ agents' policies. On the other hand, heterogeneous Proximal Policy Optimization (HAPPO)~\citep{kuba2022trust} sequentially updates agents based on their local advantages estimated from the same rollout samples (\Cref{fig:tax}b). Although it avoids the waste of collected samples and has a monotonic improvement on the joint policy, the policy improvement of a single agent is not theoretically guaranteed. Consequently, one agent's policy update may offset previous agents' policy improvement, reducing the overall joint policy improvement.

In this paper, we aim to combine the merits of the existing single rollout and sequential policy update schemes. Firstly, we show that naive sequential update algorithms with a single rollout can lose the monotonic improvement guarantee of PPO for a single agent's policy. To tackle this problem, we propose a surrogate objective with a novel off-policy correction method, \textit{\trace{}} (PreOPC), which retains the monotonic improvement guarantee on both the joint policy and each agent's policy. Then we further show that the joint monotonic bound built on the single agent bound is tighter than those of other simultaneous update algorithms and is tightened during updating the agents at a stage\footnote{We define a \textit{stage} as a period during which all the agents have been updated once (\Cref{fig:tax}).}.
This leads to
\textit{\textbf{A}gent-by-\textbf{a}gent \textbf{P}olicy \textbf{O}ptimization} (\alg{}), a novel sequential update algorithm with single rollout scheme (\Cref{fig:tax}b). 
Further, we study the significance of the agent update order and extend the theory of non-stationarity to the sequential update scheme. We test \alg{} on four popular cooperative multi-agent benchmarks: StarCraftII, multi-agent MuJoCo, multi-agent particle environment, and Google Research Football full game scenarios. On all benchmark tasks, \alg{} consistently outperforms strong baselines with a large margin in both performance and sample efficiency and shows an advantage in encouraging inter-agent coordination. To sum up, the main contributions of this work are as follows:
\begin{enumerate}[leftmargin=12pt]
    \item \textbf{Monotonic improvement bound}. We prove that the guarantees of monotonic improvement on each agent's policy could be retained under the single rollout scheme with the off-policy correction method PreOPC we proposed. We further prove that the monotonic bound on the joint policy achieved given theoretical guarantees of each agent is the tightest among single rollout algorithms, yielding effective policy optimization.
    \item \textbf{A2PO algorithm}. We propose \alg{}, the first agent-by-agent sequential update algorithm that retains the monotonic policy improvement on both each agent's policy and the joint policy and does not require multiple rollouts when performing policy improvement.
    \item \textbf{Agent update order}. We further investigate the connections between the sequential policy update scheme, the agent update order, and the non-stationarity problem, which motivates two novel methods: a semi-greedy agent selection rule for optimization acceleration and an adaptive clipping parameter method for alleviating the non-stationarity problem.
\end{enumerate}

%% file: sections/2-background_related_works.tex
\vspace{-10pt}
\section{Related Works}
\vspace{-8pt}

Trust Region Policy Optimization (TRPO) \citep{DBLP:conf/icml/SchulmanTRPO} and Proximal Policy Optimization (PPO) \citep{Schulman2017} are popular trust region algorithms with strong performances, benefiting from the guarantee of monotonic policy improvement \citep{Kakade2002}. Several recent works delve deeper into understanding these methods \citep{wang2019,Liu2019,Wang2020}.
In the multi-agent scenarios, \citet{Witt2020} and \citet{Papoudakis2020} empirically studied the performance of Independent PPO in multi-agent tasks. \citet{Yu2022} conducted a comprehensive benchmark and analyzed the factor influential to the performance of Multi-agent PPO (MAPPO), a variant of PPO with centralized critics. Coordinate PPO (CoPPO) \citep{Wu2021} integrates the value decomposition \citep{Sunehag2017} and approximately performs a joint policy improvement with monotonic improvement. Several further trials to implement trust region methods are discussed in \citet{Wen2021,Li2020,Sun2022,Ye2022}.
However, these MARL algorithms suffer from the non-stationarity problem as they update agents simultaneously. The environment dynamic changes from one agent’s perspective as others also change their policies. Consequently, agents suffer from the high variance of gradients and require more samples for convergence \citep{HernandezLeal2017}. To alleviate the non-stationarity problem, Multi-Agent Mirror descent policy algorithm with Trust region decomposition (MAMT) \citep{li2022dealing} factorizes the trust regions of the joint policy and constructs the connections among the factorized trust regions, approximately constraining the diversity of joint policy.

Rollout and Policy Iteration for a Single Agent (RPISA) \citep{Bertsekas2021} and Heterogeneous PPO (HAPPO) \citep{kuba2022trust} consider the sequential update scheme. 
RPISA suffers from sample inefficiency as it requires $n$ times of rollout for $n$ agents to complete their policies update. Additionally, their work lacks a practical algorithm for complex tasks.
In contrast, we propose a practical algorithm \alg{} that updates all agents using the same samples from a single rollout. 
HAPPO is derived from the advantage decomposition lemma, proposed as Lemma 1 in \citet{kuba2022trust}.
It does not consider the distribution shift caused by preceding agents, and has no monotonic policy improvement guarantee for each agent's policy. While \alg{} is derived without decomposing the advantage, and has a guarantee of monotonic improvement for each agent's policy. \mod{We further discuss other MARL methods in \Cref{app:others}.}

\vspace{-8pt}
\section{Trust Region Method in Sequential Policy Update Scheme}
\vspace{-5pt}
\subsection{MARL Problem Formulation}
\vspace{-5pt}
We consider formulating the sequential decision-making problem in multi-agent scenarios as a decentralized Markov decision process (DEC-MDP)~\citep{Bernstein2002}. An $n$-agent DEC-MDP can be formalized as a tuple $(\mathcal{S}, \{\mathcal{A}^{i}\}_{i \in \mathcal{N}}, r, \T, \gamma)$, where $\mathcal{N} = \{1, \ldots, n\}$ is the set of agents, $\mathcal{S}$ is the state space. $\mathcal{A}^{i}$ is the action space of agent $i$, and $\mathcal{A}=\mathcal{A}^{1}\times \cdots \times \mathcal{A}^{n}$ is the joint action space. $r: \mathcal{S}\times \mathcal{A} \mapsto  \mathbb{R}$ is the reward function, and $\T: \mathcal{S}\times \mathcal{A} \times \mathcal{S} \mapsto [0,1]$ is the dynamics function denoting the transition probability. $\gamma \in [0, 1)$ is a reward discount factor.
At time step $t$, each agent $i$ takes action $a^{i}_{t}$ from its policy $\pi^{i}(\cdot|s_t)$, simultaneously according to the state $s_t$, forming the joint action $\jntact_t=\{a_{t}^{1}, \ldots, a_{t}^{n}\}$ and the joint policy $\jntpi(\cdot|s_t)=\pi^{1}\times\ldots\times\pi^{n}$. 
The joint policy $\jntcurpi$ of these $n$ agents induces a normalized discounted state visitation distribution $d^{\jntcurpi}$, where $d^{\jntcurpi}(s)=(1-\gamma)\sum_{t=0}^{\infty}\gamma^{t}Pr(s_{t}=s|\jntcurpi)$ and $Pr(\cdot|\jntcurpi):\mathcal{S}\mapsto [0,1]$ is the probability function under $\jntcurpi$. We then define the value function $V^{\jntcurpi}(s)=\E_{\tau \sim (\T,\jntpi)} [\sum_{t=0}^\infty \gamma^{t}r(s_{t}, \bm{a_{t}})|s_{0}=s]$ and the advantage function $A^{\jntcurpi}(s,\jntact)=r(s,\jntact) + \gamma \E_{s^{\prime}\sim \mathcal{T}(\cdot|s, \jntact)}[V^{\jntcurpi}(s^{\prime})]-V^{\jntcurpi}(s)$, where $\tau = \{(s_{0}, \jntact_{0}), (s_{1}, \jntact_{1}), \ldots \}$ denotes one sampled trajectory.
The agents maximize their expected return, denoted as:
$
    \jntpi^{*} = \operatorname{argmax}_{\jntpi} \mathcal{J}(\jntpi) =\operatorname{argmax}_{\jntpi} \E_{\tau \sim (\T,\jntpi)} [\sum_{t=0}^\infty \gamma^{t}r(s_{t}, \bm{a_{t}})]~,
$.

%% file: sections/3-analysis.tex
\vspace{-5pt}
\subsection{Monotonic Improvement in Sequential Policy Update Scheme}
\vspace{-5pt}
\label{sec:monotonic_a2po}

We assume agents are updated in the order $1, 2, \ldots, n$, without loss of generality. We define $\jntcurpi$ as the joint base policy from which the agents are updated at a stage, $e^{i}=\{1, \ldots, i-1\}$ as the set of preceding agents updated before agent $i$, and $\tarpi^{i}$ as the updated policy of agent $i$. 
We denote the joint policy composed of updated policies of agents in the set $e^{i}$, the updated policy of agent $i$ and base policies of other agents as $\jntinterpi{i}=\tarpi^{1}\times\ldots\times\tarpi^{i}\times\curpi^{i+1}\times\ldots\times\curpi^{n}$, and define $\jntinterpi{0}=\jntcurpi$ and $\jntinterpi{n}=\jnttarpi$. A general sequential update scheme is shown as follows, where $\Lo{\jntinterpi{i-1}}{\jntinterpi{i}}$ is the surrogate objective for agent $i$:
\vspace{-5pt}
\begin{equation*}
    \jntcurpi =\jntinterpi{0} \xrightarrow[\text{Update } \curpi^{1}]{\max_{\curpi^{1}}{\Lo{\jntcurpi}{\jntinterpi{1}}}}
    \jntinterpi{1} \xrightarrow[]{}
    \cdots \xrightarrow[]{}
    \jntinterpi{n-1} \xrightarrow[\text{Update } \curpi^{n}]{\max_{\curpi^{n}}{\Lo{\jntinterpi{n-1}}{\jntinterpi{n}}}}
    \jntinterpi{n} = \jnttarpi.
\end{equation*}
\vspace{-5pt}

We wish our sequential update scheme retains the desired monotonic improvement guarantee while improving the sample efficiency. Before going to our method,
we first discuss \textit{\textbf{why naively updating agents sequentially with the same rollout samples will fail in monotonic improvement for each agent.}}
Since agent $i$ updates its policy from $\jntinterpi{i-1}$, an intuitive surrogate objective \citep{DBLP:conf/icml/SchulmanTRPO} used by agent $i$ could be formulated as $\mathcal{L}^{I}_{\jntinterpi{i-1}}(\jntinterpi{i}) = \J{\jntinterpi{i-1}} + \mathcal{O}_{\jntcurpi}(\jntinterpi{i})$, where $\mathcal{O}_{\jntcurpi}(\jntinterpi{i})=\frac{1}{1-\gamma}\E_{(s,\jntact) \sim (d^{\jntcurpi}, \jntinterpi{i})}[A^{\jntcurpi}(s,\jntact)]$ \mod{and the superscript $I$ means `Intuitive'}.
The expected return, however, is not guaranteed to improve with such a surrogate objective, as elaborated in the following proposition.

\vspace{-5pt}
\begin{preproposition}
    \label{prop:lost}
    For agent $i$, let $\epsilon=\max_{s, \jntact}|A^{\jntcurpi}(s, \jntact)|$, $\alpha^{j}=D_{TV}^{\max}(\curpi^{j}\|\tarpi^{j})\;\forall j \in (e^{i} \cup \{i\})$,
    where $D_{TV}(p\|q)$ is the total variation distance between distributions $p$ and $q$ and we define $D_{TV}^{\max}(\curpi\|\tarpi)=\max_{s}{D_{TV}(\curpi(\cdot|s)\|\tarpi(\cdot|s))}$,
    then we have:
    \vspace{-5pt}
    \begin{align}
        \left|
        \J{\jntinterpi{i}} - \mathcal{L}^{I}_{\jntinterpi{i-1}}(\jntinterpi{i})
        \right|
         & \leq 2 \epsilon \alpha^{i} \big(\frac{3}{1-\gamma} - \frac{2}{1-\gamma(1-\sum_{j \in (e^{i} \cup \{i\})}\alpha^{j})}\big)+
         \overbrace{\frac{2\epsilon \sum_{j\in e^{i}} \alpha^{j}}{1-\gamma}}^{\text{Uncontrollable}}
         =\beta_{i}^{I}~. 
         \label{eq:navie}
    \end{align}
\end{preproposition}
The proof can be found in \Cref{app:proof_ini}. 

\textbf{Remark}. From \Cref{eq:navie} and the definition of $\mathcal{L}^{I}_{\jntinterpi{i-1}}$, we know $\J{\jntinterpi{i}}-\J{\jntinterpi{i-1}} > \mathcal{O}_{\jntcurpi}(\jntinterpi{i}) - \beta_{i}^{I}$. Thus $\J{\jntinterpi{i}} > \J{\jntinterpi{i-1}}$ when $\mathcal{O}_{\jntcurpi}(\jntinterpi{i}) > \beta_{i}^{I}$, which can be satisfied by constraining $\beta_{i}^{I}$ and optimizing $\mathcal{O}_{\jntcurpi}(\jntinterpi{i})$. However, in $\beta_{i}^{I}$, the term $2\epsilon \sum_{j\in e^{i}} \alpha^{j}/(1-\gamma)$, is uncontrollable by agent $i$. Consequently, the upper bound $\beta_{i}^{I}$ may be large and the expected performance $\J{\jntinterpi{i}}$ may not be improved after optimizing $\mathcal{O}_{\jntcurpi}(\jntinterpi{i})$  when $\mathcal{O}_{\jntcurpi}(\jntinterpi{i}) < \beta_{i}^{I}$ even if $\alpha^{i}$ is well constrained. Although one can still prove a monotonic guarantee for the joint policy by summing \Eqref{eq:navie} for all the agents, we will show that the monotonic improvement on every single agent, if guaranteed, brings a tighter monotonic bound on the joint policy and incrementally tightens the monotonic bound on the joint policy when updating agents during a stage. Uncontrollable terms also appear when similarly analyzing HAPPO and cause the loss of monotonic improvement for a single agent\footnote{More discussions about why HAPPO fails to guarantee monotonic improvement for a single agent's policy can be found in \Cref{app:mono_baselines}.}.
 
\vspace{-8pt}
\subsection{Preceding-agent Off-policy Correction}
\vspace{-5pt}
\label{sec:preopc}

The uncontrollable term in \Cref{prop:lost} is caused by one ignoring how the updating of its preceding agents' policies influences its advantage function. We investigate reducing the uncontrollable term in policy evaluation. Since agent $i$ is updated from $\jntinterpi{i-1}$, the advantage function $A^{\jntinterpi{i-1}}$ should be used in agent $i$'s surrogate objective rather than $A^{\jntcurpi}$. However, $A^{\jntinterpi{i-1}}$ is impractical to estimate using samples collected under $\jntcurpi$ due to the off-policyness \citep{Munos2016} of these samples. Nevertheless, we can approximate $A^{\jntinterpi{i-1}}$ by correcting the discrepancy between $\jntinterpi{i-1}$ and $\jntcurpi$ at each time step \citep{Harutyunyan2016}. To retain the monotonic improvement properties, we propose \textit{\trace{}} (PreOPC), which approximates $A^{\jntinterpi{i-1}}$ using samples collected under $\jntcurpi$ by correcting the state probability at each step with truncated product weights:
\vspace{-5pt}
\begin{equation}\label{eq:trace}
    A^{\jntcurpi, \jntinterpi{i-1}}(s_{t}, \jntact_{t}) = \delta_{t} + \sum_{k\geq 1}\gamma^{k}
    \big(
    \prod_{j=1}^{k}\lambda\min\big(1.0, \frac{\jntinterpi{i-1}(\jntact_{t+j}|s_{t+j})}{\jntcurpi(\jntact_{t+j}|s_{t+j})}\big)
    \big)
    \delta_{t+k}~,
\end{equation}
where $\delta_{t} = r(s_{t},\jntact_{t}) + \gamma \V(s_{t+1}) - \V(s_{t})$ is the temporal difference for $V(s_{t})$, $\lambda$ is a parameter controlling the bias and variance, as used in \citet{Schulman2016}. $\min(1.0, \frac{\jntinterpi{i-1}(\jntact_{t+j}|s_{t+j})}{\jntcurpi(\jntact_{t+j}|s_{t+j})}) \;\forall j \in \{1,\ldots,k\}$ are truncated importance sampling weights, approximating the probability of $s_{t+k}$ at time step $t+k$ under $\jntinterpi{i-1}$. The derivation of \Cref{eq:trace} can be found in \Cref{app:preopc}.
With PreOPC, the surrogate objective of agent $i$ becomes $\Lo{\jntinterpi{i-1}}{\jntinterpi{i}}=\J{\jntinterpi{i-1}}+\frac{1}{1-\gamma}\E_{(s, \jntact)\sim (d^{\jntcurpi}, \jntinterpi{i})}[A^{\jntcurpi, \jntinterpi{i-1}}(s,\jntact)]$
, and we summarize the surrogate objective of updating all agents as follows:
\vspace{-5pt}
\begin{align}\label{eq:objective}
    \G 
    &=\J{\jntcurpi} + \frac{1}{1-\gamma}\sum_{i=1}^{n}\E_{(s, \jntact) \sim (d^{\jntcurpi}, \jntinterpi{i})}[A^{\jntcurpi, \jntinterpi{i-1}}(s,\jntact)]~.
\vspace{-3pt}
\end{align}
Note that \Eqref{eq:objective} takes the sum of expectations of the global advantage function approximated under different joint policies, different from the advantage decomposition lemma in \citet{kuba2022trust} which decomposes the global advantage function into local ones. 

We can now prove that the monotonic policy improvement guarantee of both updating one agent's policy and updating the joint policy is retained by using \Eqref{eq:objective} as the surrogate objective. The detailed proofs can be found in \Cref{app:mono_a2po}.
\vspace{-5pt}
\begin{restatable}[Single Agent Monotonic Bound]{thm}{single}
    \label{thm:single_agent_bound}
    For agent $i$, let $\epsilon^{i}=\max_{s, \jntact}|A^{\jntinterpi{i-1}}(s, \jntact)|$, $\xi^{i} = \max_{s, \jntact}|A^{\jntcurpi, \jntinterpi{i-1}}(s,\jntact)-A^{\jntinterpi{i-1}}(s,\jntact)|$, 
    $\alpha^{j}=D_{TV}^{\max}(\curpi^{j}\|\tarpi^{j})\;\forall j \in (e^{i} \cup \{i\})$, 
    then we have:
    \begin{align}
        \left|
        \J{\jntinterpi{i}} - \Lo{\jntinterpi{i-1}}{\jntinterpi{i}}
        \right|
         & \leq 4 \epsilon^{i} \alpha^{i} \big(\frac{1}{1-\gamma} - \frac{1}{1-\gamma(1-\sum_{j \in (e^{i} \cup \{i\})}\alpha^{j})}\big) + \frac{\xi^{i}}{1-\gamma} \nonumber \\
         & \leq \frac{4\gamma \epsilon^{i} }{(1-\gamma)^{2}}\big(\alpha^{i}\sum_{j \in (e^{i} \cup \{i\})}\alpha^{j}\big) + \frac{\xi^{i}}{1-\gamma}
        \label{eq:sa_bound}~.
    \end{align}    
\vspace{-5pt}
\end{restatable}
The single agent monotonic bound depends on $\epsilon^{i}$, $\xi^{i}$, and $\alpha^{i}$ and the total variation distances of preceding agents.
Unlike \Eqref{eq:navie}, we can effectively constrain the monotonic bound by controlling $\alpha^{i}$ since $\xi^{i}$ decreases as agent $i$ updating its value function \citep{Munos2016} and does not lead to an unsatisfiable bound when $\alpha^{i}$ is well constrained, providing the guarantee for monotonic improvement when updating a single agent.
Given the above bound, we can prove the monotonic improvement of the joint policy.

\input{figs/alg_table.tex}
\begin{restatable}[Joint Monotonic Bound]{thm}{joint}
    \label{thm:joint_bound}
    For each agent $i\in \mathcal{N}$, let $\epsilon^{i}=\max_{s, \jntact}|A^{\jntinterpi{i-1}}(s, \jntact)|$ , $\alpha^{i}=D_{TV}^{\max}(\curpi^{i}\|\tarpi^{i})$, $\xi^{i} = \max_{s, \jntact}|A^{\jntcurpi, \jntinterpi{i-1}}(s,\jntact)-A^{\jntinterpi{i-1}}(s,\jntact)|$, and $\epsilon = \max_{i} \epsilon^{i}$, 
    then we have:
    \begin{align}\label{eq:joint_bound}
        \left|
        \J{\jnttarpi}-\G
        \right| \nonumber
         & \leq 4 \epsilon \sum_{i=1}^{n} \alpha^{i} \big(\frac{1}{1-\gamma} - \frac{1}{1-\gamma(1-\sum_{j \in (e^{i} \cup \{i\})}\alpha^{j})}\big) +  \frac{\sum_{i=1}^{n}\xi^{i}}{1-\gamma} \nonumber \\
         & \leq \frac{4 \gamma \epsilon}{(1-\gamma)^{2}} \sum_{i=1}^{n} \big(\alpha^{i}\sum_{j \in (e^{i} \cup \{i\})}\alpha^{j} \big) +  \frac{\sum_{i=1}^{n}\xi^{i}}{1-\gamma}~.
    \end{align}
    \vspace{-15pt}
\end{restatable}

\Cref{eq:joint_bound} suggests a condition for monotonic improvement of the joint policy, similar to that in the remark under \Cref{prop:lost}. We further prove that the joint monotonic bound is incrementally tightened when performing the policy optimization agent-by-agent during a stage due to the single agent monotonic bound, i.e., the condition for improving $\J{\jnttarpi}$ is relaxed and more likely to be satisfied. The details can be found in \Cref{app:incre}. We present the monotonic bounds of other algorithms in \Cref{tab:bound}. Since $-\frac{1}{1-\gamma(1-\sum_{j \in (e^{i} \cup \{i\})}\alpha^{j})} < -\frac{1}{1-\gamma(1-\sum_{j=1}^{n}\alpha^{j})}$, \Eqref{eq:joint_bound} achieves the tightest bound compared to other single rollout algorithms, with $\xi^{i}\;\forall i \in \mathcal{N}$ small enough. The assumption about $\xi^{i}$ is valid since \trace{} is a contraction operator, which is a corollary of Theorem 1 in \citet{Munos2016}. A tighter bound improves expected performance by optimizing the surrogate objective more effectively \citep{Li2022}.

%% file: figs/alg_table.tex
\begin{table}[tbp]
    \centering
    \caption{Comparisons of trust region MARL algorithms. The proofs of the monotonic bounds can be found in \Cref{app:proof}. Note that we also provide the monotonic bound of \seqrpi{}-PPO, which implements RPISA with PPO as the base algorithm. We separate RPISA-PPO from other methods as it has low sample efficiency and thus does not constitute a fair comparison.}
    \resizebox{\linewidth}{!}{
        \begin{tabular}{ccccl}
            \toprule
            % \multirow{2}{*}{Algorithm} & \multirow{2}{*}{Rollout}  & \multirow{2}{*}{Update}     & Sample                & \multirow{2}{*}{Monotonic Bound}                                                                                                                                        \\
            % &                           &                             & Efficiency            &                                                                                                                                                                         \\
            Algorithm & Rollout  & Update     & Sample Efficiency     &  Monotonic Bound                                                                                                                                       \\
            \midrule
            \multirow{2}{*}{RPISA-PPO} & \multirow{2}{*}{Multiple} & \multirow{2}{*}{Sequential} & \multirow{2}{*}{Low}  & $4 \epsilon \sum_{i=1}^{n} \alpha^{i} (\frac{1}{1-\gamma} - \frac{1}{1-\gamma(1-\alpha^{i})})$                                                                          \\
            &                           &                             &                       & Single Agent: $4 \epsilon \alpha^{i} (\frac{1}{1-\gamma} - \frac{1}{1-\gamma(1-\alpha^{i})})$                                                                           \\
            \midrule
            MAPPO                      & Single                    & Simultaneous                & High                  & $4\epsilon \sum_{i=1}^{n}\frac{\alpha^{i}}{1-\gamma}$                                                                                                                   \\
            CoPPO                      & Single                    & Simultaneous                & High                  & $4\epsilon \sum_{i=1}^{n}\alpha^{i} (\frac{1}{1-\gamma}-\frac{1}{1-\gamma(1-\sum_{j=1}^{n}\alpha^{j})})$                                                                \\
            \multirow{2}{*}{HAPPO}     & \multirow{2}{*}{Single}   & \multirow{2}{*}{Sequential} & \multirow{2}{*}{High} & $4\epsilon \sum_{i=1}^{n}\alpha^{i} (\frac{1}{1-\gamma}-\frac{1}{1-\gamma(1-\sum_{j=1}^{n}\alpha^{j})})$                                                                \\
            &                           &                             &                       & Single Agent: No Guarantee                                                                                                                                                       \\
            \multirow{2}{*}{A2PO (ours)}      & \multirow{2}{*}{Single}   & \multirow{2}{*}{Sequential} & \multirow{2}{*}{High} & $4 \epsilon \sum_{i=1}^{n} \alpha^{i} (\frac{1}{1-\gamma} - \frac{1}{1-\gamma(1-\sum_{j \in (e^{i} \cup \{i\})}\alpha^{j})}) +  \frac{\sum_{i=1}^{n}\xi^{i}}{1-\gamma}$ \\
            &                           &                             &                       & Single Agent: $4 \epsilon^{i} \alpha^{i} (\frac{1}{1-\gamma} - \frac{1}{1-\gamma(1-\sum_{j \in (e^{i} \cup \{i\})}\alpha^{j})}) + \frac{\xi^{i}}{1-\gamma}$             \\
            \bottomrule
        \end{tabular}
        }
        \label{tab:bound}
\vspace{-10pt}
\end{table}

%% file: sections/4-method.tex
\vspace{-10pt}
\section{Agent-by-agent Policy Optimization}
\vspace{-5pt}
\label{sec:a2po}

We first give a practical implementation for optimizing the surrogate objective $\G$. When updating agent $i$, the monotonic bound in \Eqref{eq:sa_bound} consists of the total variation distances related to the preceding agents and agent $i$, i.e., $\alpha^{i}\sum_{j \in (e^{i} \cup \{i\})}\alpha^{j}$. It suggests that we can control the monotonic bound by controlling total variation distances $\alpha^{j}\; \forall j \in (e^{i} \cup \{i\})$, to effectively improve the expected performance. We consider applying the clipping mechanism to control the total variation distances $\alpha^{j}\; \forall j \in (e^{i} \cup \{i\})$ \citep{Queeney2021,Sun2022}. In the surrogate objective of agent $i$, i.e., $\J{\jntinterpi{i-1}}+\frac{1}{1-\gamma}\E_{(s, \jntact)\sim (d^{\jntcurpi}, \jntcurpi)}[\frac{\tarpi^{i}\prod_{j\in e^{i}}\tarpi^{j}}{\curpi^{i}\prod_{j\in e^{i}}\curpi^{j}}A^{\jntcurpi, \jntinterpi{i-1}}(s,\jntact)]$, $\J{\jntinterpi{i-1}}$ has no dependence to agent $i$, while the joint policy ratio $\frac{\tarpi^{i}\prod_{j\in e^{i}}\tarpi^{j}}{\curpi^{i}\prod_{j\in e^{i}}\curpi^{j}}$ in the advantage estimation is appropriate for applying the clipping mechanism. We further consider reducing the instability in estimating agent $i$'s policy gradient by clipping the joint policy ratio of preceding agents first, with a narrower clipping range \citep{Wu2021}. Thus we apply the clipping mechanism on the joint policy ratio twice: once on the joint policy ratio of preceding agents and once on the policy ratio of agent $i$. Finally, the practical objective for updating agent $i$ becomes:
\begin{equation}\label{eq:obj}
    \tilde{\mathcal{L}}_{\jntinterpi{i-1}}(\jntinterpi{i}) = \E_{(s, \jntact) \sim (d^{\jntcurpi}, \jntcurpi)}\big[
        \min
        \big(
        l(s,\jntact)A^{\jntcurpi, \jntinterpi{i-1}},
        \operatorname{clip}\big(l(s,\jntact), 1\pm\epsilon^{i} \big)A^{\jntcurpi, \jntinterpi{i-1}}
        \big)
        \big]~,
\end{equation}
where $l(s,\jntact)=\frac{\tarpi^{i}(a^{i}|s)}{\curpi^{i}(a^{i}|s)} g(s, \jntact)$, and $g(s, \jntact) = \operatorname{clip}(\frac{\prod_{j \in e^{i}}\tarpi^{j}(a^{j}|s)}{\prod_{j \in e^{i}}\curpi^{j}(a^{j}|s)}, 1\pm\frac{\epsilon^{i}}{2})$. The clipping parameter $\epsilon^{i}$ is selected as $\epsilon^{i}=\mathcal{C}(\epsilon, i)$, where $\epsilon$ is the base clipping parameter \mod{and $\mathcal{C}(\cdot,\cdot)$ is the clipping parameter adapting function}. We summarize our proposed \textit{\textbf{A}gent-by-\textbf{a}gent \textbf{P}olicy \textbf{O}ptimization} (\alg{}) in \Cref{alg:a2poframework}. Note that in \Cref{alg_line:selection}, the agent for the next update iteration is selected according to the agent selection rule $\mathcal{R}(\cdot)$.

\vspace{-5pt}
\begin{algorithm}[h]
    \caption{Agent-by-agent Policy Optimization (A2PO)}
    \label{alg:a2poframework}
    Initialize the joint policy $\jntcurpi_{0}=\{\curpi^{1}_{0}, \ldots, \curpi^{n}_{0}\}$, and the global value function $\V$.

    \For{iteration $m=1,2,\ldots$} {
    Collect data using $\jntcurpi_{m-1}=\{\curpi^{1}_{m-1}, \ldots, \curpi^{n}_{m-1}\}$.

    \For{Order $k=1,\ldots,n$} {

    Select an agent according to the selection rule as $i=\mathcal{R}(k)$. \\
    Policy $\curpi_{m}^{i}=\curpi^{i}_{m-1}$, preceding agents $e^{i}=\{\mathcal{R}(1), \ldots, \mathcal{R}(k-1)\}$. \label{alg_line:selection}

    Joint policy $\jntinterpi{i}=\{\curpi_{m}^{i}, {\curpi_{m}^{j \in e^{k}}}, {\curpi^{j \in {\mathcal{N} - e^{k}}}_{m-1}}\}$.\\
    Compute the advantage approximation as $A^{\jntcurpi, \jntinterpi{i-1}}(s, \jntact)$ via \Eqref{eq:trace}.

    Compute the value target $v(s_t) = A^{\jntcurpi, \jntinterpi{i-1}}(s, \jntact) + \V(s)$.

    \For{$P$ epochs} {
    $\curpi^{i}_{m} = \argmax_{\curpi^{i}_{m}} \tilde{\mathcal{L}}_{\jntinterpi{i-1}}(\jntinterpi{i})$ as in \Eqref{eq:obj}.

    $\V = \argmin_{\V} \E_{s\sim d^{\jntcurpi}}\|v(s)-V(s)\|^{2}$.\\
    }

    }

    }
\end{algorithm}
\vspace{-5pt}
\Eqref{eq:obj} approximates the surrogate objective of a single agent. \mod{We remark that the monotonic improvement guarantee of a single agent reveals how the update of a single agent affects the overall objective.} We will further discuss $\mathcal{R}(\cdot)$ and $\mathcal{C}(\cdot,\cdot)$ 
from the perspective of how to benefit the optimization of the overall surrogate objective by coordinating the policy updates of each agent.

\textbf{Semi-greedy Agent Selection Rule}. With the monotonic policy improvement guarantee on the joint policy, as shown in \Cref{thm:joint_bound}, we can effectively improve the expected performance $\J{\jnttarpi}$ by optimizing the surrogate objective of all agents $\G=\J{\jntcurpi} + \sum_{i=1}^{n}\Lo{\jntinterpi{i-1}}{\jntinterpi{i}}$. Since the policies except $\curpi^{i}$ are fixed when maximizing $\Lo{\jntinterpi{i-1}}{\jntinterpi{i}}$, we recognize maximizing $\sum_{i=1}^{n} \tilde{\mathcal{L}}_{\jntinterpi{i-1}}(\jntinterpi{i})$ as performing a block coordinate ascent, i.e., iteratively seeking to update a block of chosen coordinates (agents) while other blocks (agents) are fixed. As a special case of the coordinate selection rule, the agent selection rule becomes crucial for convergence. On the one hand, intuitively, updating agent with a bigger absolute value of the advantage function contributes more to optimizing $\G$. Inspired by the Gauss-Southwell rule \citep{Gordon2015}, we propose the greedy agent selection rule, under which an agent with a bigger absolute value of the expected advantage function is updated with a higher priority. 
We will verify that the agents with small absolute values of the advantage function also benefit from the greedy selelction rule in \Cref{app:add_exp_ablation}. 
On the other hand, purely greedy selection may lead to early convergence which harms the performance. Therefore, we introduce randomness into the agent selection rule to avoid converging too early \citep{Lu2018}. Combining the merits, we propose the semi-greedy agent selection rule as $\begin{cases}
        \mathcal{R}(k) = \argmax_{i \in (\mathcal{N} - e)} \E_{s,a^{i}}[|A^{\jntcurpi, \jntinterpi{\mathcal{R}(k-1)}}|], & k \mod 2 = 0 \\
        \mathcal{R}(k) \sim \mathcal{U}(\mathcal{N} - e),                                                   & k \mod 2 = 1
    \end{cases}$, where $e = \{\mathcal{R}(1), \ldots, \mathcal{R}(k-1)\}$ and $\mathcal{U}$ is a uniform distribution. We verify that the semi-greedy agent selection rule contributes to the performance of A2PO in \Cref{subs:ablation}.

\textbf{Adaptive Clipping Parameter}. We improve the sample efficiency by updating all agents using the samples collected under the base joint policy $\jntcurpi$. However, when updating agent $i$ by optimizing $\frac{1}{1-\gamma}\E_{(s, \jntact)\sim (d^{\jntcurpi}, \jntinterpi{i})}[A^{\jntcurpi, \jntinterpi{i-1}}(s,\jntact)]$, the expectation of advantage function is estimated using the states sampled under $\jntcurpi$ instead of $\jntinterpi{i-1}$, which reintroduces the non-stationarity since agent $i$ can not perceive the change of the preceding agents. With the non-stationarity modeled by the state transition shift \citep{Sun2022}, we define the state transition shift encountered when updating agent $i$ as $\Delta_{\curpi^{1}, \ldots, \curpi^{n}}^{\tarpi^{1}, \ldots, \tarpi^{i-1}, \curpi^{i}, \ldots, \curpi^{n}}(s^{\prime}|s) = \sum_{\jntact}[\T(s^{\prime}|s,\jntact)(\jntinterpi{i-1}(\jntact|s)-\jntcurpi(\jntact|s))]$. The state transition shift has the following property.

\vspace{-5pt}
\begin{preproposition}
    \label{prop:shift}
    The state transition shift $\Delta_{\curpi^{1}, \ldots, \curpi^{n}}^{\tarpi^{1}, \ldots, \tarpi^{i-1}, \curpi^{i}, \ldots, \curpi^{n}}(s^{\prime}|s)$ can be decomposed as follows.
    \[
        \Delta_{\curpi^{1}, \ldots, \curpi^{n}}^{\tarpi^{1}, \ldots, \tarpi^{i-1}, \curpi^{i}, \ldots, \curpi^{n}} =
        \Delta_{\pi^{1}, \ldots, \pi^{n}}^{\tarpi^{1}, \pi^{2}, \ldots, \pi^{n}}
        + \Delta_{\tarpi^{1}, \pi^{2}, \ldots, \pi^n}^{\tarpi^{1}, \tarpi^{2}, \curpi^{3}, \ldots, \pi^n}
        + \cdots
        + \Delta_{\tarpi^{1}, \ldots, \tarpi^{i-2}, \curpi^{i-1}, \ldots, \pi^n}^{\tarpi^{1}, \ldots, \tarpi^{i-1}, \curpi^{i}, \ldots, \pi^{n}}
        \]
\end{preproposition}
\vspace{-5pt}

\Cref{prop:shift} shows that the total state transition shift encountered by agent $i$ can be decomposed into the sum of state transition shift caused by each agent whose policy has been updated. Shifts caused by agents with higher priorities will be encountered by more following agents and thus contribute more to the non-stationarity problem.
Recall that the state transition shift effectively measures the total variation distance between policies. Therefore, in order to reduce the non-stationarity brought by the agents' policy updates, we can adaptively clip each agent's surrogate objective according to their update priorities. We propose a simple yet effective method, named adaptive clipping parameter, to adjust the clipping parameters according to the updating order: $\mathcal{C}(\epsilon, k) = \epsilon \cdot c_{\epsilon} + \epsilon \cdot (1-c_{\epsilon}) \cdot k / n$, where $c_{\epsilon}$ is a hyper-parameter. We demonstrate how the agents with higher priorities affect the following agents in \Cref{fig:cpt_ana}. Under the clipping mechanism, the influence of the agents with higher priority could be reflected in the clipping ranges of the joint policy ratio. The policy changes of the preceding agents may constrain the following agents to optimize the surrogate objective within insufficient clipping ranges, as shown on the left side of \Cref{fig:cpt_ana}. The right side of \Cref{fig:cpt_ana} demonstrates that the adaptive clipping parameter method leads to balanced and sufficient clipping ranges.

\begin{figure}[htbp]
    \centering
    \includegraphics[width=0.49\linewidth]{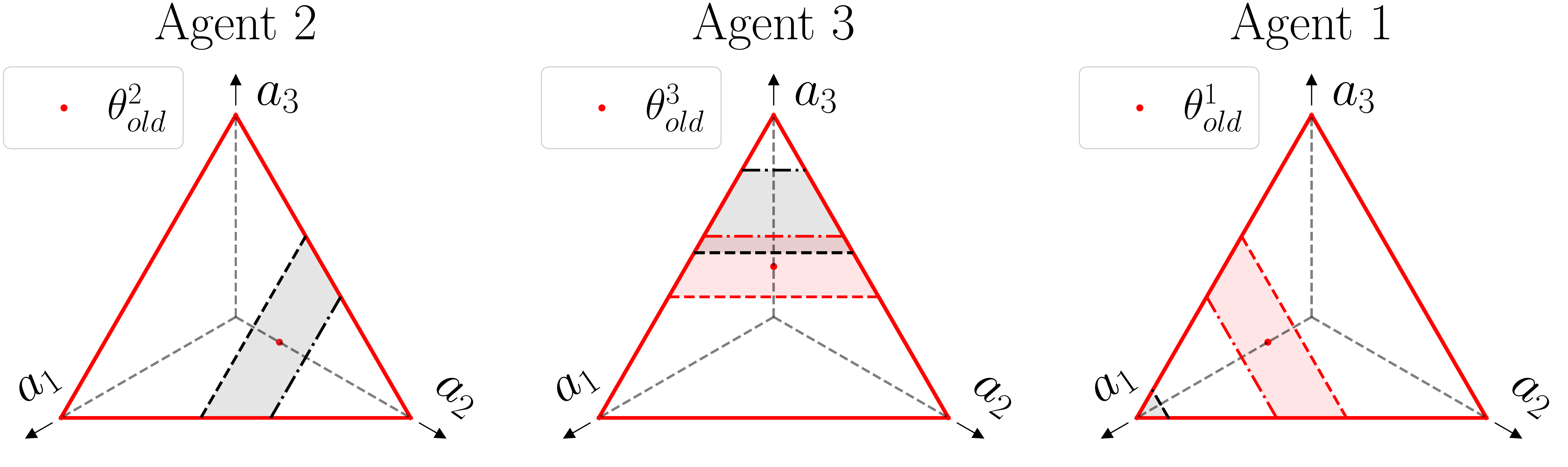}
    \hfill
    \tikz{\draw[-,black, densely dashed, thick](0,-1.0) -- (0,1.0);}
    \hfill
    \includegraphics[width=0.49\linewidth]{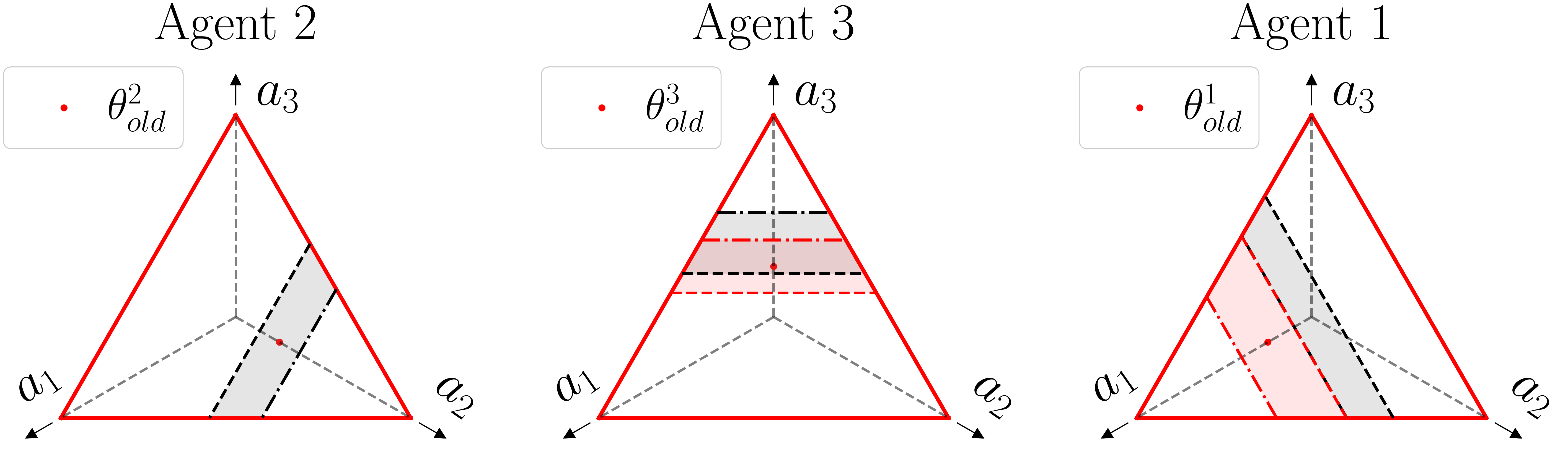}
    \caption{The clipping ranges of three agents. The surface $a_{1}+a_{2}+a_{3}=1$ demonstrates the policy space of three discrete actions. The agents are updated in the order of $2,3,1$. The areas in gray/pink are the clipping ranges with/without considering the joint policy ratio of preceding agents. \textbf{Left}: The agents have the same clipping parameters. The clipping range of agent \mod{$1$} is insufficient due to the large variation in the policies of \mod{agent $2$ and agent $3$}. \textbf{Right}: The clipping ranges are more balanced and sufficient with the adaptive clipping parameter method.}
    \label{fig:cpt_ana}
\end{figure}

%% file: sections/5-experiments.tex
\vspace{-10pt}
\section{Experiments}
\vspace{-5pt}

In this section, we empirically evaluate and analyze \alg{} in the widely adopted cooperative multi-agent benchmarks, including the StarCraftII Multi-agent Challenge (SMAC) \citep{Samvelyan2019}, Multi-agent MuJoCo (MA-MuJoCo) \citep{Witt2020}, Multi-agent Particle Environment (MPE) \citep{Lowe2017}\footnote{We evaluate \alg{} in fully cooperative and general-sum MPE tasks respectively, showing the potential of extending \alg{} to general-sum games, see \Cref{app:exp_setup_result_mpe} for full results.}, and more challenging Google Research Football (GRF) full-game scenarios \citep{Kurach2020}. Experimental results demonstrate that \textbf{1)} \textit{\alg{} achieves performance and efficiency superior to those of state-of-the-art MARL Trust Region methods}, \textbf{2)} \textit{\alg{} has strength in encouraging coordination behaviors to complete complex cooperative tasks}, and \textit{\textbf{3)} the PreOPC, the semi-greedy agent selection rule, and the adaptive clipping parameter methods significantly contribute to the performance improvement}. \footnote{Code is available at \url{https://anonymous.4open.science/r/A2PO}.}

We compare \alg{} with advanced MARL trust-region methods: MAPPO \citep{Yu2022}, CoPPO \citep{Wu2021} and HAPPO \citep{kuba2022trust}. We implement all the algorithms as parameter sharing in SMAC and MPE, and as parameter-independent in MA-MuJoCo and GRF, according to the homogeneity and heterogeneity of agents. We divide the agents into blocks for tasks with numerous agents to control the training time of \alg{} comparable to other algorithms. Full experimental details can be found in \Cref{app:exp}.
\vspace{-5pt}
\subsection{Performance and Efficiency}
\vspace{-5pt}

We evaluate the algorithms in 9 maps of SMAC with various difficulties, 14 tasks of 6 scenarios in MA-MuJoCo, and the 5-vs-5 and 11-vs-11 full game scenarios in GRF. Results in \Cref{tab:smac}, \Cref{fig:mujoco}, and \Cref{fig:grf} show that \alg{} consistently outperforms the baselines and achieves higher sample efficiency in all benchmarks. More results and the experimental setups can be found in \Cref{app:exp_setup_result}.

\textbf{StarCraftII Multi-agent Challenge (SMAC)}. As shown in \Cref{tab:smac}, \alg{} achieves (nearly) 100\% win rates in 6 out of 9 maps and significantly outperforms other baselines in most maps. \mod{In \Cref{tab:smac}, we additionally compare the performance with that of Qmix \citep{rashid2018qmix}, a well known baseline in SMAC}. We also observe that CoPPO and \alg{} have better stability as they consider clipping joint policy ratios. 
\input{figs/results-SMAC.tex}
\textbf{Multi-agent MuJoCo environment (MA-MuJoCo)}. We investigate whether \alg{} can scale to more complex continuous control multi-agent tasks in MA-MuJoCo.
We calculate the normalized score $\frac{\text{return} - \text{minimum return}}{\text{maximum return} - \text{minimum return}}$ over all the 14 tasks in the left of \Cref{fig:mujoco}. 
We also present part of results in the right of \Cref{fig:mujoco}, where the control complexity and observation dimension, depending on the number of the robot's joints, increases from left to right. We observe that \alg{} generally shows an increasing advantage over the baselines with increasing task complexity.

\begin{figure}[htbp]
    \centering
    \includegraphics[height=0.2\linewidth]{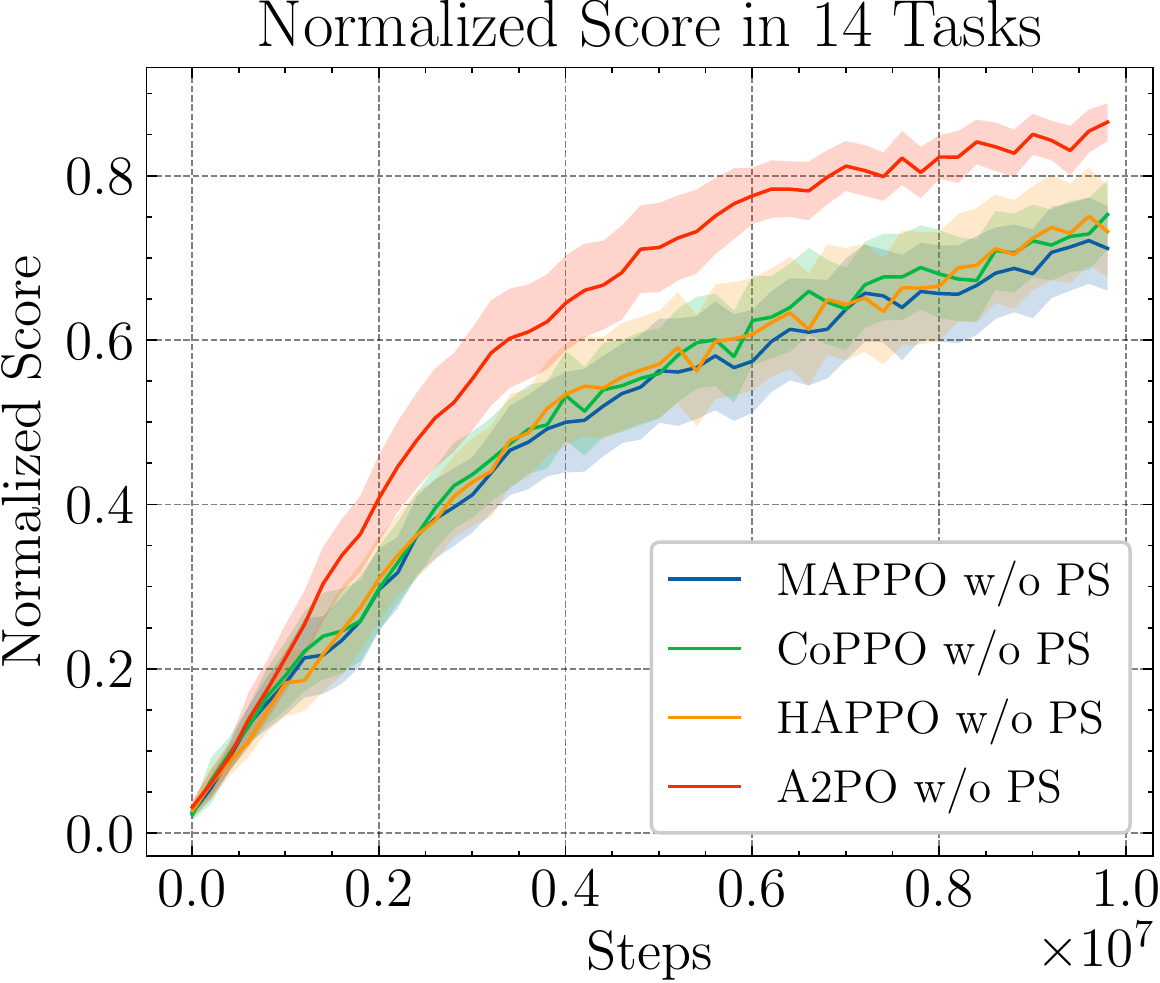}
    \hfill
    \tikz{\draw[-,black, solid, thick](0,-0.8) -- (0,1.9);}
    \hfill
    \includegraphics[height=0.2\linewidth]{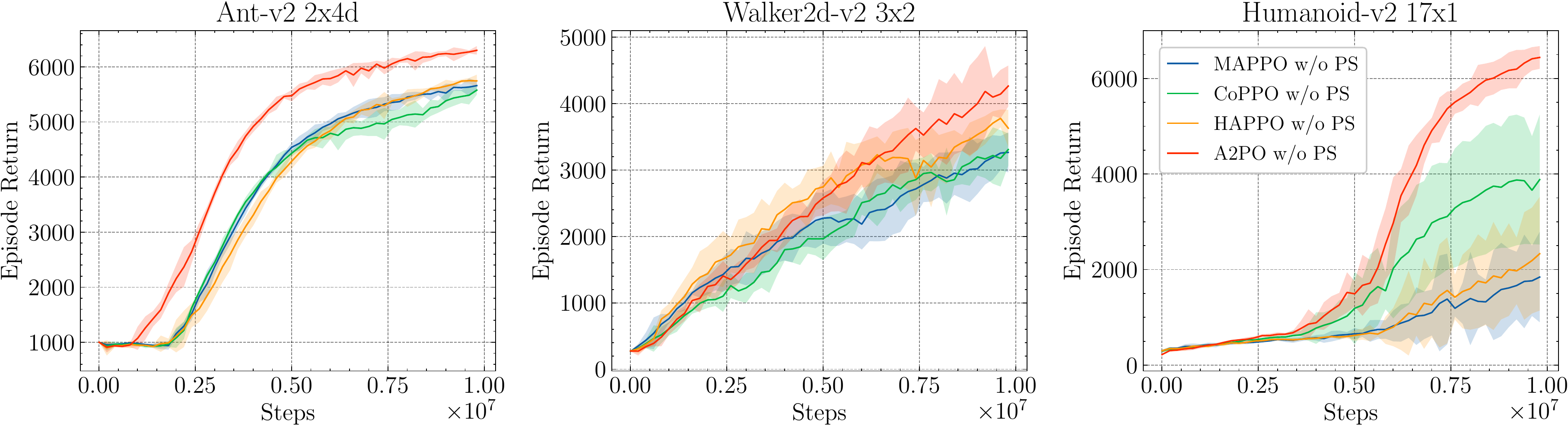}
    \caption{Experiments in MA-MuJoCo. \textbf{Left}: Normalized scores on all the 14 tasks. \textbf{Right}: Comparisons of averaged return on selected tasks. The number of robot joints increases from left to right.}
    \label{fig:mujoco}
\end{figure}
\vspace{-5pt}
\textbf{Google Research Football (GRF)}. 
We evaluate \alg{} in GRF full-game scenarios, where agents have difficulty discovering complex coordination behaviors. 
\alg{} obtains nearly 100\% win rate in the 5-vs-5 scenario. In both scenarios, we attribute the performance gain of \alg{} to the learned coordination behavior. 
We analyze the experiments in GRF to verify that \alg{} encourages agents to learn coordination behaviors in complex tasks. 
In \Cref{tab:grf_coordination}, an `Assist' is attributed to the player who passes the ball to the teammate that makes a score, a `Pass' is counted when the passing-and-receiving process is finished, `Pass Rate' is the proportion of success passes over the pass attempts.
\alg{} have an advantage in passing-and-receiving coordination, leading to more assists and scores. 

\begin{figure}[htbp]
    \centering
    \begin{minipage}[t]{0.48\linewidth}%
        \vspace{0pt}
        \centering
        \includegraphics[width=\linewidth]{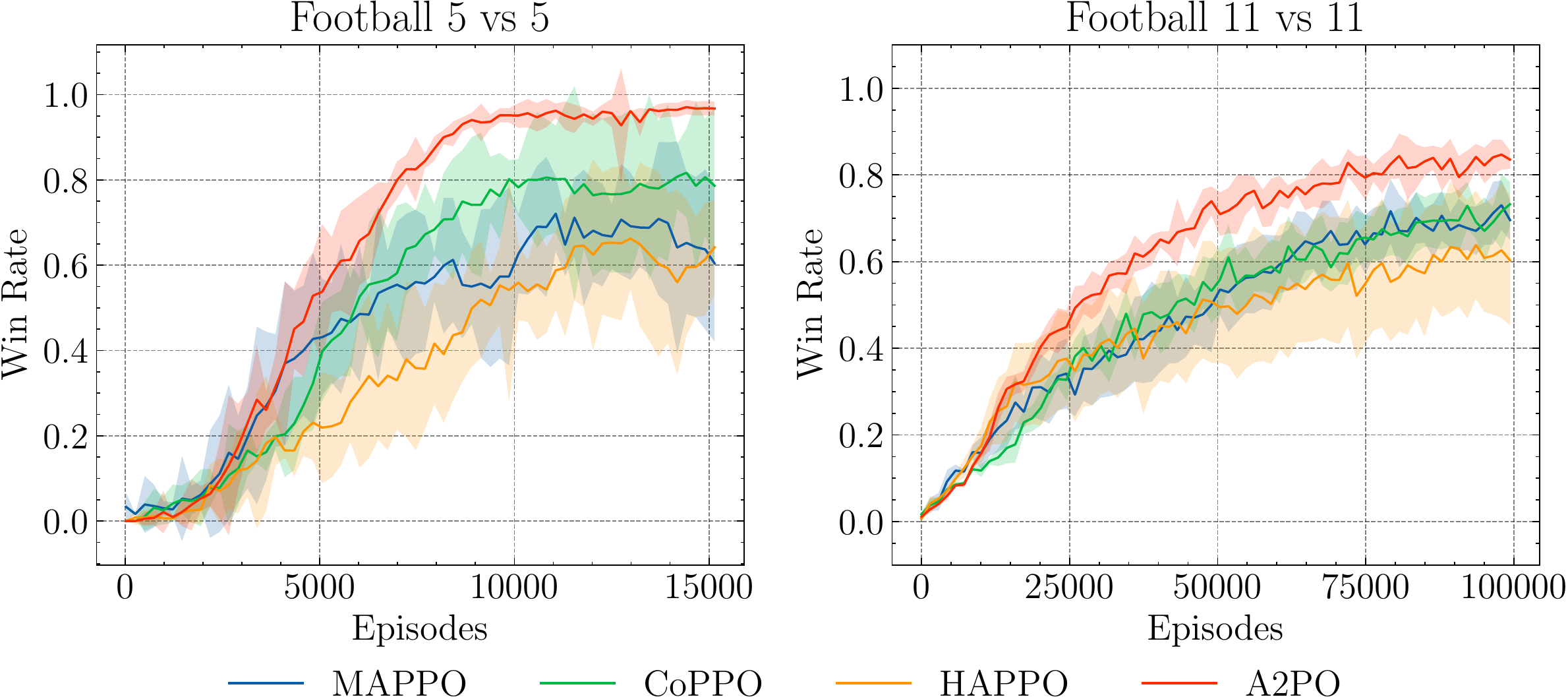}
        \caption{Averaged win rate on the Google Research Football full-game scenarios.}
        % \vspace{0pt}
        \label{fig:grf}
    \end{minipage}
    \hfill
    \begin{minipage}[t]{0.48\linewidth}
        \centering
        \vspace{0pt}
        \captionof{table}{Learned behaviors on the Google Research Football 5-vs-5 scenario. Bigger values are better except fot the `Lost' metric.}
        \label{tab:grf_coordination}
        \resizebox{\linewidth}{!}{
        \renewcommand\arraystretch{1.0}  
        \begin{tabular}{lllll}
        \toprule
         Metric         & MAPPO      & CoPPO      & HAPPO      & A2PO       \\
        \midrule
         Assist            & 0.04\tiny(0.02) & 0.19\tiny(0.08) & 0.07\tiny(0.05) & \textbf{0.56\tiny(0.20)} \\
         Goal              & 1.95\tiny(1.17)     & 4.42\tiny(2.08)    & 2.68\tiny(0.86)    & \textbf{9.01\tiny(0.95)}  \\
         Lost               & \textbf{0.49\tiny(0.11)}     & 0.74\tiny(0.33)    & 1.04\tiny(0.12)     & 0.78\tiny(0.15)   \\
        %  Goal Diff.   & 1.46\tiny(1.09)     & 3.68\tiny(1.92)    & 1.63\tiny(0.79)    & \textbf{8.22\tiny(1.01) }  \\
         Pass      & 1.52\tiny(0.13)     & 3.44\tiny(1.04)    & 4.03\tiny(1.97)     & \textbf{6.42\tiny(2.23) }  \\
         Pass Rate & 19.3\tiny(10.0)       & 35.0\tiny(10.3)      & 48.9\tiny(25.7)     & \textbf{67.1\tiny(11.7)}   \\
        \bottomrule
        \end{tabular}
        }
        % \vspace{0pt}
    \end{minipage}
\end{figure}

\vspace{-5pt}
\subsection{Ablation Study}
\label{subs:ablation}

This section studies how PreOPC, the semi-greedy agent selection rule, and the adaptive clipping parameter affect the performance.
% and analyze how sensitive \alg{} is to the key hyper-parameters
% We conduct ablation experiments on representative tasks of SMAC and MA-MuJoCo. 
Full ablation details can be found in \Cref{app:add_exp_ablation}

\begin{wrapfigure}{r}{0.5\linewidth}
    \centering
    \vspace{-10pt}
    \includegraphics[width=1.0\linewidth]{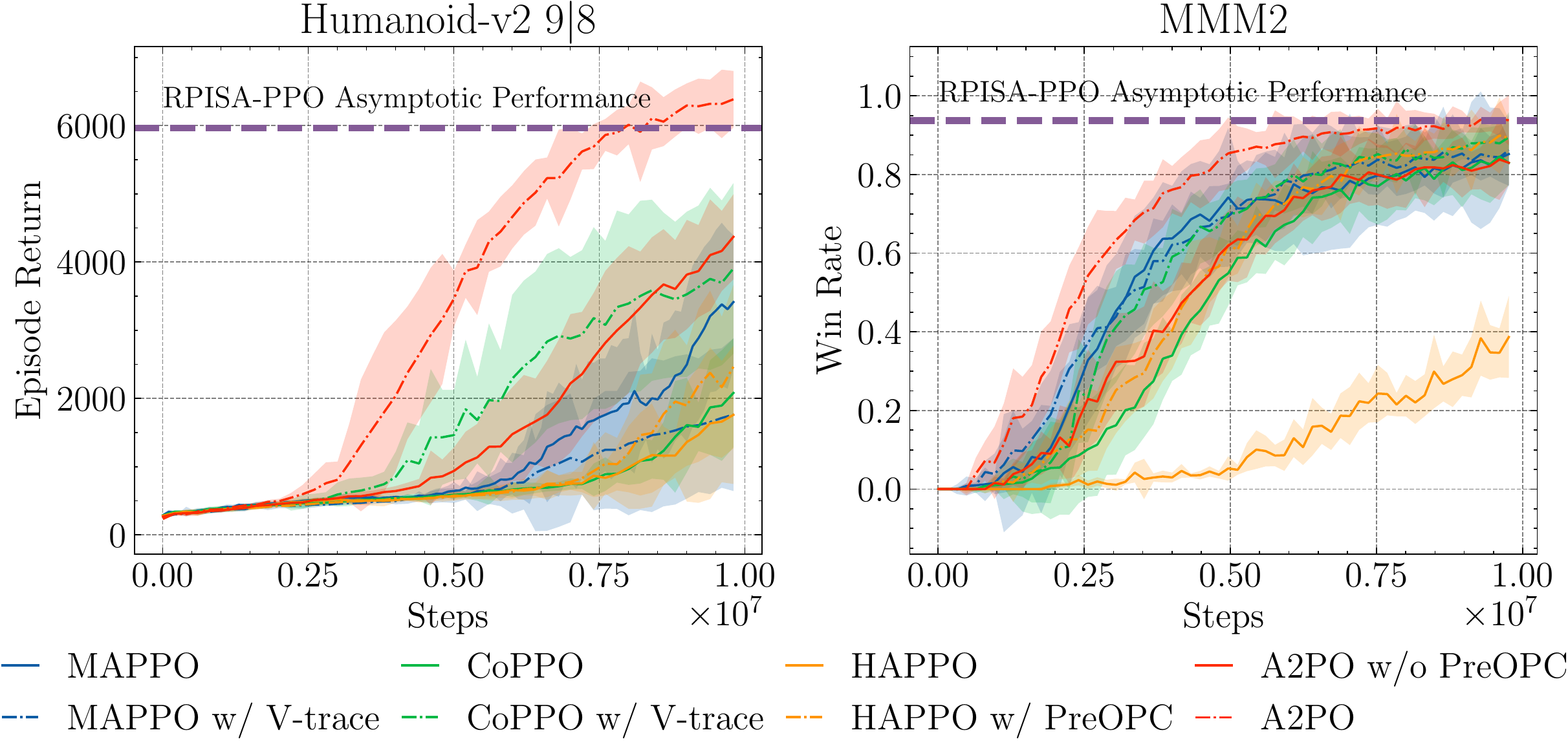}
    \caption{Ablation experiments on \trace{}.}
    \label{fig:ablation_trace}
    \vspace{-5pt}
\end{wrapfigure}
\textbf{PreOPC}. \Cref{fig:ablation_trace} shows the effects of utilizing off-policy correction in two cases: 1) Correction on all agents' policies for simultaneous update algorithms, i.e., MAPPO w/ V-trace \mod{\citep{Espeholt2018}} and CoPPO w/ V-trace, and 2) Correction on the preceding agents' policies for sequential update algorithms, i.e., HAPPO w/ PreOC and \alg{}. V-trace brings no general improvement to MAPPO and CoPPO, while PreOPC significantly improves the sequential update cases. PreOPC improves the performance of HAPPO significantly, while A2PO still outperforms HAPPO w/ PreOPC. The performance gap lies in that \alg{} clips the joint policy ratios, which matches the monotonic bound in \Cref{thm:single_agent_bound}. 
The results verify that \alg{} reaches or outperforms the asymptotic performance of \seqrpippo{} using an approximated advantage function and updating all the agents with the same rollout samples. 
Additionally, \trace{} does not increase the sensitivity of the hyper-parameter $\lambda$, as shown in \Cref{app:add_exp_ablation}.

\textbf{Agent Selection Rule}. We provide comparisons of different agent selection rules in \Cref{fig:ablation_seq_strategy}. The `Cyclic' rule means select agents in the order $1,\ldots, n$, and other rules have been introduced in \cref{sec:a2po}. The semi-greedy rule considers the optimization acceleration and the performance balance among agents and thus performs the best in all tasks.

\textbf{Adaptive Clipping Parameter}. We propose the adaptive clipping parameter method for balanced and sufficient clipping ranges of agents. As shown in \Cref{fig:ablation_cpt}, the adaptive clipping parameter contributes to the performance gain of \alg{}.

\begin{figure}[htbp]
    \centering
    \begin{minipage}[t]{0.48\linewidth}
        \centering
        \includegraphics[width=1.0\linewidth]{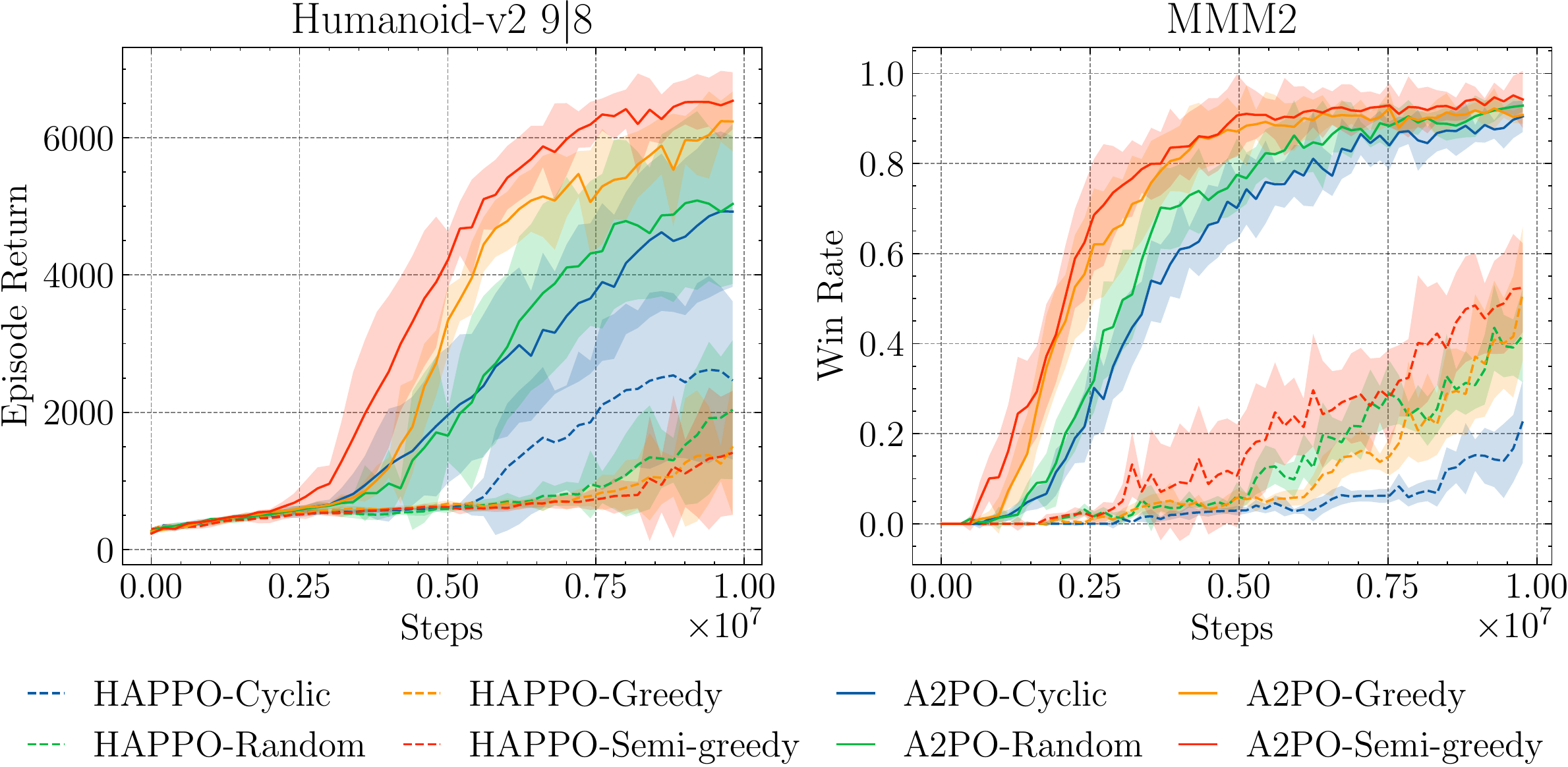}
        \caption{Ablation experiments on the agent selection rules.}
        \label{fig:ablation_seq_strategy}
    \end{minipage}
    \hfill
    \begin{minipage}[t]{0.48\linewidth}%
        \centering
        \includegraphics[width=1.0\linewidth]{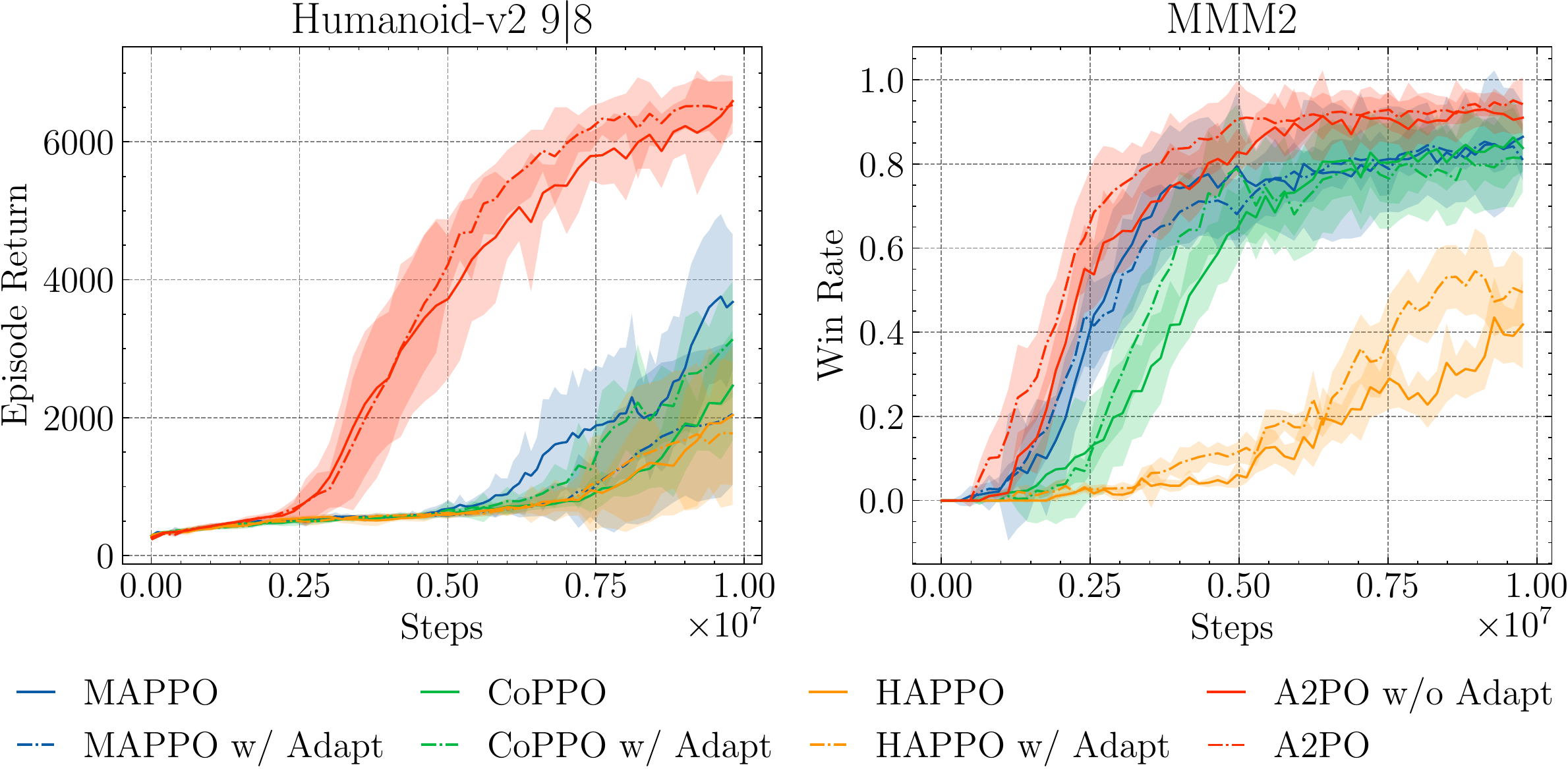}
        \caption{Ablation experiments on the adaptive clipping parameter method.}
        \label{fig:ablation_cpt}
    \end{minipage}
\end{figure}

%% file: figs/results-SMAC.tex
\begin{table}[htbp]
\vspace{-5pt}
    \centering
    \caption{Median win rates and standard deviations on SMAC tasks.}
    \small
    \begin{tabular}{llllll|l}
        \toprule
        Map            & Difficulty & MAPPO w/ PS         & CoPPO w/ PS          & HAPPO w/ PS         & A2PO w/ PS   & Qmix w/ PS     \\
        \midrule
        MMM            & Easy       & 96.9\tiny(0.988)         & 96.9\tiny(1.25)           & 95.3\tiny(2.48)          & \textbf{100\tiny(1.07)}  & 95.3\tiny(5.2) \\
        % 3s\_vs\_5z     & Hard       & \textbf{100\tiny(1.17)}  & \textbf{100\tiny(2.08)}   & \textbf{100\tiny(0.659)} & \textbf{100\tiny(0.534)} \\
        % 2c\_vs\_64zg   & Hard       & \textbf{98.4\tiny(1.74)} & 96.9\tiny(0.521)          & 96.9\tiny(0.521)         & 96.9\tiny(0.659)         \\
        3s5z           & Hard       & 84.4\tiny(4.39)          & 92.2\tiny(2.35)           & 92.2\tiny(1.74)          & \textbf{98.4\tiny(1.04)} & 88.3\tiny(2.9) \\
        5m\_vs\_6m     & Hard       & 84.4\tiny(2.77)          & 84.4\tiny(2.12)           & 87.5\tiny(2.51)          & \textbf{90.6\tiny(3.06)} & 75.8\tiny(3.7) \\
        8m\_vs\_9m     & Hard       & 84.4\tiny(2.39)          & 84.4\tiny(2.04)           & 96.9\tiny(3.78)          & \textbf{100\tiny(1.04)} & 92.2\tiny(2.0)  \\
        10m\_vs\_11m   & Hard       & 93.8\tiny(18.7)          & 96.9\tiny(2.6)            & 98.4\tiny(2.99)          & \textbf{100\tiny(0.521)} & 95.3\tiny(1.0) \\
        6h\_vs\_8z     & Super Hard & 87.5\tiny(1.53)          & \textbf{90.6\tiny(0.765)} & 87.5\tiny(1.49)          & \textbf{90.6\tiny(1.32)} & 9.4\tiny(2.0)\\
        3s5z\_vs\_3s6z & Super Hard & 82.8\tiny(19.2)          & 84.4\tiny(2.9)            & 37.5\tiny(13.2)          & \textbf{93.8\tiny(19.8)} & 82.8\tiny(5.3) \\
        MMM2           & Super Hard & 90.6\tiny(8.89)          & 90.6\tiny(6.93)           & 51.6\tiny(9.01)          & \textbf{98.4\tiny(1.25)} & 87.5\tiny(2.6) \\
        27m\_vs\_30m   & Super Hard & 93.8\tiny(3.75)          & 93.8\tiny(2.2)            & 90.6\tiny(4.77)          & \textbf{100\tiny(1.55)} & 39.1\tiny(9.8)  \\
        % corridor       & Super Hard & 96.9\tiny(0)             & \textbf{100\tiny(0.659)}  & 96.9\tiny(0.96)          & \textbf{100\tiny(0)}     \\
        \midrule
        Overall      & /            & 88.7\tiny(6.96)  & 90.5\tiny(2.57)  & 81.9\tiny(4.67)    & \textbf{96.9\tiny(3.41)} & 74.0\tiny(3.83) \\
        \bottomrule
    \end{tabular}
    % \vspace{-5pt}
    \label{tab:smac}
\end{table}

%% file: sections/6-conclusion.tex
\vspace{-5pt}
\section{Conclusion}
In this paper, we investigate the potential of the sequential update scheme in coordination tasks. We introduce \alg{}, a sequential algorithm using a single rollout at a stage, which guarantees monotonic improvement on both the joint policy and each agent's policy. We also justify that the monotonic bound achieved by \alg{} is the tightest among existing trust region MARL algorithms under single rollout scheme. Furthermore, \alg{} integrates the proposed semi-greedy agent selection rule and adaptive clipping parameter method. Experiments in various benchmarks demonstrate that \alg{} consistently outperforms state-of-the-art methods in performance and sample efficiency and encourages coordination behaviors for completing complex tasks. For future work, we plan to analyze the theoretical underpinnings of the agent selection rules and study the learnable methods to select agents and clipping parameters.

%% file: sections/appendix-proof.tex
\section{Proofs}
\label{app:proof}

\subsection{Notations}

We list the main notations used in \Cref{tab:notation}.

\begin{table}[h]
    \centering
    \caption{The notations and symbols used in this paper.}
    \small
    \begin{tabular}{c|l}
        \toprule
        Notation & Definition \\
        \midrule
        $\mathcal{S}$ & The state space \\
        $\mathcal{N}$ & The set of agents \\
        $n$ & The number of agents \\
        $i$ & The agent index \\
        $\mathcal{A}^{i}$ & The action space of agent $i$ \\
        $r$ & The reward function \\
        $\mathcal{T}$ & The transition function \\
        $\gamma$ & The discount factor \\
        $t$ & The time-step \\
        $s_{t}$ & The state at time-step $t$ \\
        $a^{i}_{t}$ & The action of agent $i$ at time-step $t$ \\
        $\bm{a}_{t}$ & The joint action at time-step $t$ \\
        $d^{\jntcurpi}$ & The discounted state visitation distribution \\
        $Pr$ & The state probability function \\
        $V$ & The value function \\
        $A$ & The advantage function \\
        $\tau$ & The trajectory of an episode \\
        $e$ & A set of preceding agents \\
        $e^{i}$ & The set of preceding agents updated before agent $i$ \\
        $\pi^{i}$ & The policy of agent $i$ \\
        $\tarpi^{i}$ & The updated policy of agent $i$ \\
        $\bm{\pi}$ & The joint policy \\
        $\lambda$ & The bias and variance balance parameter \\
        $\jnttarpi$ & The joint target policy \\
        $\jntinterpi{i}$ & The joint policy after updating agent $i$ \\
        $\mathcal{J}(\jntcurpi)$ & The expected return / performance of the joint policy $\jntcurpi$ \\
        $\mathcal{L}_{\jntinterpi{i-1}}(\jntinterpi{i})$ & The surrogate objective of agent $i$ \\
        $\mathcal{L}_{\jntinterpi{i-1}}^{I}(\jntinterpi{i})$ & An intuitive surrogate objective of agent $i$ \\
        $\mathcal{G}_{\jntcurpi}(\jnttarpi)$ & The surrogate objective of all agents \\
        $\epsilon$ & The upper bound of an advantage function \\
        $D_{TV}$ & The total variation distance function \\ 
        $\alpha$ & The total variation distance between 2 policies \\
        $\xi^{i}$ & The off policy correction error of $\jntinterpi{i-1}$ \\
        $\mathcal{C}$ & The clipping parameter adaptation function \\
        $\mathcal{R}$ & The agent selection function \\
        \bottomrule
    \end{tabular}
    \label{tab:notation}
\end{table}

\subsection{Useful Lemmas}
\begin{lemma}
    \label{lemma:pdl_joint}
    (Multi-agent Policy Performance Difference Lemma). Given any joint policies $\jnttarpi$ and $\jntcurpi$, the difference between the performance of two joint policies can be expressed as:
    \[
        \J{\jnttarpi} - \J{\jntcurpi} =  \frac{1}{1-\gamma}\E_{(s, \jntact)\sim (d^{\jnttarpi}, \jnttarpi)}\left[A^{\jntcurpi}(s, \jntact)\right]~,
    \]
    where $d^{\curpi}=(1-\gamma)\sum_{t=0}^{\infty}\gamma^{t}Pr(s_{t}=s|\curpi)$ is the normalized discounted state visitation distribution.
\end{lemma}

\begin{proof}
    A corollary of the Policy Performance Difference Lemma, see Lemma 1.16 in \citet{Agarwal2022}.
\end{proof}

For convenience, we give some properties and definitions of coupling\footnote{The definition of coupling and the properties can be found in any textbook containing Markov Chains.} and the definition of $\alpha$-coupled policy pair \citep{DBLP:conf/icml/SchulmanTRPO} here.

\begin{definition}{Coupling}
    A coupling of two probability distributions $\mu$ and $\nu$ is a pair of random variables $(X, Y)$ such that the marginal distribution of $X$ is $\mu$ and the marginal distribution of $Y$ is $\nu$. A coupling $(X, Y)$ satisfies the following constraints: $Pr(X = x) = \mu(x)$ and $Pr(Y = y) = \nu(y)$.
\end{definition}

\begin{preproposition}
    For any coupling $(X, Y)$ that $D_{TV}(\mu\|\nu) \leq Pr(X\neq Y)$.
\end{preproposition}

\begin{preproposition}\label{prop:0}
    There exists a coupling $(X, Y)$ that $D_{TV}(\mu\|\nu) = Pr(X\neq Y)$.
\end{preproposition}

\begin{corollary}\label{cor:1}
    For all $s$, there exists a coupling ($\jntcurpi(\cdot|s), \jnttarpi(\cdot|s)$), that $Pr(\jntact = \jntactb) \geq 1 - D_{TV}^{max}(\jntcurpi\|\jnttarpi)$, for $\jntact \sim \jntcurpi(\cdot|s)$, $\jntactb \sim \jnttarpi(\cdot|s)$.
\end{corollary}
\begin{proof}
    By \cref{prop:0} there exists a coupling ($\jntcurpi(\cdot|s), \jnttarpi(\cdot|s)$), s.t.
    \begin{equation*}
        1 - Pr(\jntact = \jntactb) = Pr(\jntact \ne \jntactb) = D_{TV}(\jntcurpi, \jnttarpi) \le D_{TV}^{max}(\jntcurpi \| \jnttarpi)
    \end{equation*}
\end{proof}
\begin{corollary}\label{cor:2}
    For all $s$, $D_{TV}(\jntcurpi(\cdot|s)\|\jnttarpi(\cdot|s)) \leq \sum_{i=1}^{n} D_{TV}(\curpi^{i}(\cdot|s)\|\tarpi^{i}(\cdot|s))$.
\end{corollary}

\begin{proof}
    We denote $\pi(\cdot|s)$ as $\pi(\cdot)$ for brevity.
    \begin{align*}
             & D_{TV}(\jntcurpi(\cdot|s)\|\jnttarpi(\cdot|s))                                                                                                                                                                                                                                                        \\
        =    & \frac{1}{2}\sum_{a^{1},a^{2},\ldots,a^{n}}\left|\prod_{i=1}^{n}\pi^{i}(a^{i})-\prod_{i=1}^{n}\tarpi^{i}(a^{i})\right|                                                                                                                                                                                 \\
        =    & \frac{1}{2}\sum_{a^{1},a^{2},\ldots,a^{n}}\left|\prod_{i=1}^{n}\pi^{i}(a^{i}) - \pi^{1}(a^{1})\prod_{i=2}^{n}\tarpi^{i}(a^{i})+\pi^{1}(a^{1})\prod_{i=2}^{n}\tarpi^{i}(a^{i})-\prod_{i=1}^{n}\tarpi^{i}(a^{i})\right|                                                                                 \\
        \leq & \frac{1}{2} \sum_{a^{1}}\left|\pi^{1}(a^{1})\right|\sum_{a^{2},\ldots,a^{n}}\left|\prod_{i=2}^{n}\pi^{i}(a^{i})-\prod_{i=2}^{n}\tarpi^{i}(a^{i})\right| + \frac{1}{2} \sum_{a^{1}}\left|\pi^{1}(a^{1})-\tarpi^{1}(a^{1})\right|\sum_{a^{2},\ldots,a^{n}}\left|\prod_{i=2}^{n}\tarpi^{i}(a^{i})\right| \\
        =    & \frac{1}{2} \sum_{a^{2},\ldots,a^{n}}\left|\prod_{i=2}^{n}\pi^{i}(a^{i})-\prod_{i=2}^{n}\tarpi^{i}(a^{i})\right| + \frac{1}{2}\sum_{a^{1}}\left|\pi^{1}(a^{1})-\tarpi^{1}(a^{1})\right|                                                                                                               \\
             & \cdots                                                                                                                                                                                                                                                                                                \\
        \leq & \frac{1}{2} \sum_{i=1}^{n} \sum_{a^{i}}|\pi^{i}(a^{i})-\tarpi^{i}(a^{i})|                                                                                                                                                                                                                             \\
        =    & \sum_{i=1}^{n} D_{TV}(\curpi^{i}(\cdot|s)\|\tarpi^{i}(\cdot|s))
    \end{align*}
\end{proof}

\begin{definition}{$\alpha$-coupled policy pair}
    If $(\jntcurpi, \jnttarpi)$ is an $\alpha$-coupled policy pair, then $(\jntact, \jntactb|s)$ satisfies $Pr(\jntact\neq \jntactb|s) \leq \alpha$ for all $s$, and $\jntact \sim \jntcurpi(\cdot|s)$, $\jntactb \sim \jnttarpi(\cdot|s)$.
\end{definition}

From \Cref{cor:1,cor:2}, we know that given any joint policy pair $\jntcurpi$ and $\jnttarpi$, select $\alpha=D_{TV}^{max}(\jntcurpi(\cdot|s)\|\jnttarpi(\cdot|s))$, then ($\jntcurpi,\jnttarpi$) is an $\alpha$-coupled policy pair that for all $s$, $Pr(\jntact\neq \jntactb|s) \leq D_{TV}^{max}(\jntcurpi(\cdot|s)\|\jnttarpi(\cdot|s)) \leq \sum_{i=1}^{n}\alpha^{i}$, where $\alpha^{i}=D_{TV}^{max}(\curpi^{i}\|\tarpi^{i})$.

\begin{lemma}
    \label{lemma:adv_up_bound}
    Given any joint policies $\jnttarpi$ and $\jntcurpi$, the following inequality holds:
    \[
        |\E_{\jntact \sim \jnttarpi}\left[A^{\jntcurpi}(s,\jntact)\right]|
        \leq 2\epsilon \sum_{i=1}^{n} \alpha^{i}~,
    \]
    where $\alpha^{i}=D_{TV}^{max}(\tarpi^{i}\|\curpi^{i})$ and $\epsilon = \max_{s, \jntact} |A^{\jntcurpi}(s,\jntact)|$.
\end{lemma}

\begin{proof}
    Note that $\E_{\jntact \sim \jntcurpi}[A^{\jntcurpi}(s,\jntact)] = 0$.
    We have
    \begin{align*}
        \left|\E_{\jntact \sim \jnttarpi}\left[A^{\jntcurpi}(s,\jntact)\right]\right| & = \left|\E_{\jntactb \sim \jnttarpi}\left[A^{\jntcurpi}(s,\jntactb)\right] - \E_{\jntact \sim \jntcurpi}\left[A^{\jntcurpi}(s,\jntact)\right] \right|            \\
                                                                                      & = \left|\E_{(\jntactb, \jntact)\sim (\jnttarpi, \jntcurpi)}\left[A^{\jntcurpi}(s,\jntactb) - A^{\jntcurpi}(s,\jntact)\right] \right|                             \\
                                                                                      & = \left| Pr(\jntactb \neq \jntact|s)\E_{(\jntactb, \jntact)\sim (\jnttarpi, \jntcurpi)}\left[A^{\jntcurpi}(s,\jntactb) - A^{\jntcurpi}(s,\jntact)\right] \right| \\
                                                                                      & \leq \sum_{i=1}^{n}\alpha^{i}\E_{(\jntactb, \jntact)\sim (\jnttarpi, \jntcurpi)}\left[\left|A^{\jntcurpi}(s,\jntactb) - A^{\jntcurpi}(s,\jntact)\right|\right]   \\
                                                                                      & \leq \sum_{i=1}^{n}\alpha^{i} \cdot 2 \max_{s, \jntact} |A^{\jntcurpi}(s,\jntact)|
    \end{align*}
\end{proof}

\begin{lemma}\label{lemma:genralized_adv_discrepancy_bound}
    (Multi-agent Advantage Discrepancy Lemma). Given any joint policies $\jntcurpi^{1}$, $\jntcurpi^{2}$ and $\jntcurpi^{3}$, the following inequality holds:
    \begin{align*}
             & \left|\mathbb{E}_{(s_{t}, \jntact_{t})\sim (Pr^{\jntcurpi^{2}}, \jntcurpi^{2})}\left[A^{\jntcurpi^{1}} \right] - \mathbb{E}_{(s_{t}, \jntactb_{t})\sim (Pr^{\jntcurpi^{3}}, \jntcurpi^{2})}\left[A^{\jntcurpi^{1}} \right]\right| \\
        \leq &
        4\epsilon^{\jntcurpi^{1}} \cdot D_{TV}^{max}(\jntcurpi^{1}\|\jntcurpi^{2}) \cdot (1 - {(1-D_{TV}^{max}(\jntcurpi^{2}\|\jntcurpi^{3}))}^{t})~,
    \end{align*}
    where $\epsilon^{\jntcurpi^{1}} = \max_{s,\jntact} \|A^{\jntcurpi^{1}}(s,\jntact)\|$ and we denote $A(s, \jntact)$ as $A$ for brevity.
\end{lemma}

\begin{proof}
    Let $n_{t}$ represent the times $\jntact \neq \jntactb$ ($\jntcurpi^{1}$ disagrees with $\jntcurpi^{3}$) before timestamp $t$.
    \begin{align*}
                            & \left|\mathbb{E}_{(s_{t}, \jntact_{t})\sim (Pr^{\jntcurpi^{2}}, \jntcurpi^{2})}\left[A^{\jntcurpi^{1}} \right] - \mathbb{E}_{(s_{t}, \jntactb_{t})\sim (Pr^{\jntcurpi^{3}}, \jntcurpi^{2})}\left[A^{\jntcurpi^{1}} \right]\right|                                         \\
        =                   & Pr(n_{t} > 0) \cdot \left|\mathbb{E}_{(s_{t}, \jntact_{t})\sim (Pr^{\jntcurpi^{2}}, \jntcurpi^{2})|n_{t} > 0}\left[A^{\jntcurpi^{1}} \right] - \mathbb{E}_{(s_{t}, \jntactb_{t})\sim (Pr^{\jntcurpi^{3}}, \jntcurpi^{2})|n_{t} > 0}\left[A^{\jntcurpi^{1}} \right]\right| \\
        \overset{(a)}{=}    & (1-Pr(n_{t}=0)) \cdot E                                                                                                                                                                                                                                                   \\
        \leq                & (1 - \prod_{k=1}^{t}Pr(\jntact_{k}=\jntactb_{k}|\jntact_{k}\sim \jntcurpi^{2}(\cdot|s_k), \jntactb_{k}\sim \jntcurpi^{3}(\cdot|s_k))) \cdot E                                                                                                                             \\
        \overset{(b)}{\leq} & (1 - \prod_{k=1}^{t}(1-D_{TV}^{max}(\jntcurpi^{2}\|\jntcurpi^{3}))) \cdot E                                                                                                                                                                                               \\
        =                   & (1 - {(1-D_{TV}^{max}(\jntcurpi^{2}\|\jntcurpi^{3}))}^{t}) \cdot E                                                                                                                                                                                                        \\
        \leq                & (1 - {(1-D_{TV}^{max}(\jntcurpi^{2}\|\jntcurpi^{3}))}^{t}) \cdot 2 \cdot 2 \cdot D_{TV}^{max}(\jntcurpi^{1}\|\jntcurpi^{2}) \cdot \epsilon^{\jntcurpi^{1}}                                                                                                                \\
        =                   & 4\epsilon^{\jntcurpi^{1}} \cdot D_{TV}^{max}(\jntcurpi^{1}\|\jntcurpi^{2}) \cdot (1 - {(1-D_{TV}^{max}(\jntcurpi^{2}\|\jntcurpi^{3}))}^{t})
    \end{align*}
    In (a), we denote $|\mathbb{E}_{(s_{t}, \jntact_{t})\sim (Pr^{\jntcurpi^{2}}, \jntcurpi^{2})|n_{t} > 0}[A^{\jntcurpi^{1}} ] - \mathbb{E}_{(s_{t}, \jntactb_{t})\sim (Pr^{\jntcurpi^{3}}, \jntcurpi^{2})|n_{t} > 0}[A^{\jntcurpi^{1}} ]|$ as $E$ for brevity. (b) follows the definition of $\alpha$-coupled policy pair.
\end{proof}

We provide a useful equation of the normalized discounted state visitation distribution here.
\begin{preproposition}
    \label{prop:d_pr}
    \begin{align*}
        \E_{(s, \jntact) \sim (d^{\jntcurpi^{1}}, \jntcurpi^{2})}\left[f(s, \jntact)\right]
         & = (1-\gamma)\sum_{s}\sum_{t=0}^{\infty}\gamma^{t}Pr(s_{t}=s|\jntcurpi^{1})\sum_{\jntact}\jntcurpi^{2}(\jntact|s)f(s,\jntact)       \\
         & = (1-\gamma) \sum_{t=0}^{\infty} \gamma^{t} \sum_{s}Pr(s_{t}=s|\jntcurpi^{1})\sum_{\jntact}\jntcurpi^{2}(\jntact|s)f(s,\jntact)    \\
         & = (1-\gamma) \sum_{t=0}^{\infty} \gamma^{t} \E_{(s_{t}, \jntact_{t}) \sim (Pr^{\jntcurpi^{1}}, \jntcurpi^2)}[f(s_{t},\jntact_{t})]
    \end{align*}
\end{preproposition}

\subsection{Proofs of Intuitive Sequential Update}
\label{app:proof_ini}

\begin{align*}
    & \left|\J{\jntinterpi{i}}-\J{\jntinterpi{i-1}}-\frac{1}{1-\gamma}\E_{(s,\jntact)\sim (d^{\jntcurpi}, \jntinterpi{i})}\left[A^{\jntcurpi}\right]\right|                                                                                                             \\
\leq   & \frac{1}{1-\gamma}\left| \E_{(s,\jntact)\sim (d^{\jntinterpi{i}}, \jntinterpi{i})}\left[A^{\jntinterpi{i-1}}\right] -\E_{(s,\jntact)\sim (d^{\jntcurpi}, \jntinterpi{i})}\left[A^{\jntcurpi}\right] \right|                                                     \\
\leq    & \frac{1}{1-\gamma}\left| \E_{(s,\jntact)\sim (d^{\jntinterpi{i}}, \jntinterpi{i})}\left[A^{\jntinterpi{i-1}}\right] - \E_{(s,\jntact)\sim (d^{\jntcurpi}, \jntinterpi{i})}\left[A^{\jntinterpi{i-1}}\right]\right|                                                                  \\
    & + \frac{1}{1-\gamma}\left|\E_{(s,\jntact)\sim (d^{\jntcurpi}, \jntinterpi{i})}\left[A^{\jntinterpi{i-1}}\right] -\E_{(s,\jntact)\sim (d^{\jntcurpi}, \jntinterpi{i})}\left[A^{\jntcurpi}\right] \right|                                                       \\
\leq & 4 \epsilon^{\jntinterpi{i-1}} \alpha^{i}\sum_{t=0}^{\infty} \gamma^{t} (1-(1-\sum_{j \in (e^{i} \cup \{i\})}\alpha^{j})^{t})  \\
& + \frac{1}{1-\gamma}\E_{(s,\jntact)\sim (d^{\jntcurpi}, \jntinterpi{i})}\left[\left|A^{\jntinterpi{i-1}}-A^{\jntcurpi}\right|\right] \\
\leq & 4 \epsilon^{\jntinterpi{i-1}} \alpha^{i} (\frac{1}{1-\gamma} - \frac{1}{1-\gamma(1-\sum_{j \in (e^{i} \cup \{i\})}\alpha^{j})})   + \frac{1}{1-\gamma} \left[4\alpha^{i}\epsilon^{\jntinterpi{i-1}}+2\sum_{j \in e^{i}}\alpha^{j}\epsilon^{\jntcurpi}\right] \\
\end{align*}

\subsection{Proofs of Monotonic Policy Improvement of A2PO}
\label{app:mono_a2po}

\single*

\begin{proof}
    Using \Cref{lemma:genralized_adv_discrepancy_bound} and \Cref{prop:d_pr}, we get
    \begin{align*}
             & \left|\J{\jntinterpi{i}}-\J{\jntinterpi{i-1}}-\frac{1}{1-\gamma}\E_{(s,\jntact)\sim (d^{\jntcurpi}, \jntinterpi{i})}\left[A^{\jntcurpi, \jntinterpi{i-1}}\right]\right|                                                                                                               \\
        =    & \frac{1}{1-\gamma}\left| \E_{(s,\jntact)\sim (d^{\jntinterpi{i}}, \jntinterpi{i})}\left[A^{\jntinterpi{i-1}}\right] -\E_{(s,\jntact)\sim (d^{\jntcurpi}, \jntinterpi{i})}\left[A^{\jntcurpi, \jntinterpi{i-1}}\right] \right|                                                         \\
        \leq & \frac{1}{1-\gamma}\left| \E_{(s,\jntact)\sim (d^{\jntinterpi{i}}, \jntinterpi{i})}\left[A^{\jntinterpi{i-1}}\right] - \E_{(s,\jntact)\sim (d^{\jntcurpi}, \jntinterpi{i})}\left[A^{\jntinterpi{i-1}}\right]\right|                                                                    \\
             & + \frac{1}{1-\gamma}\left|\E_{(s,\jntact)\sim (d^{\jntcurpi}, \jntinterpi{i})}\left[A^{\jntinterpi{i-1}}\right] -\E_{(s,\jntact)\sim (d^{\jntcurpi}, \jntinterpi{i})}\left[A^{\jntcurpi, \jntinterpi{i-1}}\right] \right|                                                             \\
        \leq & 4 \epsilon^{\jntinterpi{i-1}} \alpha^{i}\sum_{t=0}^{\infty} \gamma^{t} (1-(1-\sum_{j \in (e^{i} \cup \{i\})}\alpha^{j})^{t})   + \frac{1}{1-\gamma}\E_{(s,\jntact)\sim (d^{\jntcurpi}, \jntinterpi{i})}\left[\left|A^{\jntinterpi{i-1}}-A^{\jntcurpi, \jntinterpi{i-1}}\right|\right] \\
        \leq & 4 \epsilon^{\jntinterpi{i-1}} \alpha^{i} (\frac{1}{1-\gamma} - \frac{1}{1-\gamma(1-\sum_{j \in (e^{i} \cup \{i\})}\alpha^{j})})   + \frac{1}{1-\gamma} \xi^{i}                                                                                                                        \\
        % = & 4 \epsilon^{\jntinterpi{i-1}} \alpha^{i} (\frac{1+\xi^{i}}{1-\gamma} - \frac{1}{1-\gamma(1-\sum_{j \in (e^{i} \cup \{i\})}\alpha^{j})}) 
    \end{align*}
\end{proof}

\joint*

\begin{proof}
    \begin{align*}
             & \left|
        \J{\jnttarpi}-\G
        \right| \nonumber
        \\
        =    & \left|
        \J{\jnttarpi}-\J{\jntcurpi}-\sum_{i=1}^{n}\E_{(s, \jntact)\sim (d^{\jntcurpi}, \jntinterpi{i})}
        \left[A^{\jntcurpi, \jntinterpi{i-1}}(s,\jntact)\right] \right|
        \\
        =    & \left|
        \J{\jntinterpi{n}} - \J{\jntinterpi{n-1}} + \cdots + \J{\jntinterpi{1}} - \J{\jntinterpi{0}} - \frac{1}{1-\gamma} \sum_{i=1}^{n}\E_{(s, \jntact)\sim (d^{\jntcurpi}, \jntinterpi{i})}
        \left[A^{\jntcurpi, \jntinterpi{i-1}}(s,\jntact)\right]
        \right| \nonumber                                                                                                                                                                                 \\
        \leq & \sum_{i=1}^{n} \left| \J{\jntinterpi{i}} - \J{\jntinterpi{i-1}} - \frac{1}{1-\gamma}\E_{(s, \jntact)\sim (d^{\jntcurpi}, \jntinterpi{i})}
        \left[A^{\jntcurpi, \jntinterpi{i-1}}(s,\jntact)\right]  \right| \nonumber                                                                                                                        \\
        \leq & 4 \epsilon \sum_{i=1}^{n} \alpha^{i} \left(\frac{1}{1-\gamma} - \frac{1}{1-\gamma(1-\sum_{j \in (e^{i} \cup \{i\})}\alpha^{j})}\right) +  \frac{\sum_{i=1}^{n}\xi^{i}}{1-\gamma} \nonumber \\
        \leq & \frac{4 \gamma \epsilon}{(1-\gamma)^{2}} \sum_{i=1}^{n} \left(\alpha^{i}\sum_{j \in (e^{i} \cup \{i\})}\alpha^{j} \right) +  \frac{\sum_{i=1}^{n}\xi^{i}}{1-\gamma}~.
    \end{align*}
\end{proof}

\subsection{Proofs of Incrementally Tightened Bound of A2PO}
\label{app:incre}

Assume agent $k$ is updated with order $k$ in the sequence $1, \ldots, n$, since $\jntinterpi{k-1}$ is known, we have

\begin{align*}
     &  \left|
    \J{\jnttarpi}-\G
   \right| \nonumber \\
\leq & \sum_{i=1}^{k-1} \left| \J{\jntinterpi{i}} - \Lo{\jntinterpi{i-1}}{\jntinterpi{i}} \right| + 4 \epsilon \sum_{i=k}^{n}\alpha^{i} \left(\frac{1}{1-\gamma} - \frac{1}{1-\gamma(1-\sum_{j \in (e^{i} \cup \{i\})}\alpha^{j})}\right)+\frac{\sum_{i=k}^{n}\xi^{i}}{1-\gamma} \\
\leq & \sum_{i=1}^{k-2} \left| \J{\jntinterpi{i}} - \Lo{\jntinterpi{i-1}}{\jntinterpi{i}} \right| + 4 \epsilon \sum_{i=k-1}^{n} \alpha^{i} \left(\frac{1}{1-\gamma} - \frac{1}{1-\gamma(1-\sum_{j \in (e^{i} \cup \{i\})}\alpha^{j})}\right) +  \frac{\sum_{i=k-1}^{n}\xi^{i}}{1-\gamma} \\
& \vdots \\
\leq &  4 \epsilon \sum_{i=1}^{n} \alpha^{i} \left(\frac{1}{1-\gamma} - \frac{1}{1-\gamma(1-\sum_{j \in (e^{i} \cup \{i\})}\alpha^{j})}\right) +  \frac{\sum_{i=1}^{n}\xi^{i}}{1-\gamma} 
\end{align*}

Thus the condition for improving $\J{\jnttarpi}$ is relaxed during updating agents at a stage.

\subsection{Proofs of Monotonic Policy Improvement of MAPPO, CoPPO and HAPPO}
\label{app:mono_baselines}

In this section, we give proof of the monotonic policy improvement of MAPPO, and unify the formats of the monotonic bounds of CoPPO and HAPPO, without considering the parameter-sharing method.

\textbf{MAPPO}. For MAPPO, $\Lo{\jntcurpi}{\jnttarpi} = \sum_{i=1}^{n}\J{\jntcurpi} +
    \frac{1}{1-\gamma}\left[\E_{(s,\jntact)\sim (d^{\jntcurpi}, \jntcurpi)}\left[\frac{\tarpi^{i}}{\curpi^{i}}A^{\jntcurpi}\right]\right]$. We first prove that for agent $i$, $\J{\jnttarpi}-\J{\jntcurpi}-\frac{1}{1-\gamma}\left[\E_{(s,\jntact)\sim (d^{\jntcurpi}, \jntcurpi)}\left[\frac{\tarpi^{i}}{\curpi^{i}}A^{\jntcurpi}\right]\right]$ is bounded.

\begin{align*}
         & \left|\J{\jnttarpi}-\J{\jntcurpi}-\frac{1}{1-\gamma}\left[\E_{(s,\jntact)\sim (d^{\jntcurpi}, \jntcurpi)}\left[\frac{\tarpi^{i}}{\curpi^{i}}A^{\jntcurpi}\right]\right]\right|                                                      \\
    =    & \frac{1}{1-\gamma} \left|\E_{(s,\jntact)\sim(d^{\jnttarpi}, \jnttarpi)}\left[A^{\jntcurpi}\right]-\E_{(s,\jntact)\sim(d^{\jntcurpi}, \jntcurpi)}\left[\frac{\tarpi^{i}}{\curpi^{i}}A^{\jntcurpi}\right]\right|                      \\
    =    & \sum_{t=0}^{\infty}\gamma^{t} \left| \E_{(s_{t},\jntact_{t})\sim(Pr^{\jnttarpi}, \jnttarpi)} A^{\jntcurpi} - \E_{(s_{t},\jntact_{t})\sim(Pr^{\jntcurpi}, \jntcurpi)}\left[\frac{\tarpi^{i}}{\curpi^{i}}A^{\jntcurpi}\right] \right| \\
    \leq & \sum_{t=0}^{\infty} 2\gamma^{t} \left(\left(\sum_{j=1}^{n}\alpha^{j}\right)\cdot \epsilon^{\jntcurpi}+\alpha^{i} \cdot\epsilon^{\jntcurpi} \right)                                                                                  \\
    =    & \frac{2\epsilon^{\jntcurpi}}{1-\gamma}\left(\alpha^{i} + \sum_{j=1}^{n}\alpha^{j} \right)
\end{align*}

Sum the bounds for all agents and take the average, we get

\begin{align*}
         & \left|\J{\jnttarpi}-\J{\jntcurpi}-\frac{1}{n}\frac{1}{1-\gamma}\sum_{i=1}^{n}\left[\E_{(s,\jntact)\sim (d^{\jntcurpi}, \jntcurpi)}\left[\frac{\tarpi^{i}}{\curpi^{i}}A^{\jntcurpi}\right]\right]\right| \\
    \leq & \frac{2\epsilon^{\jntcurpi}}{1-\gamma} \frac{n+1}{n} \sum_{j=1}^{n}\alpha^{j}
\end{align*}

Finally, the monotonic bound for MAPPO is

\begin{align*}
         & \left|\J{\jnttarpi}-\J{\jntcurpi}-\frac{1}{1-\gamma}\sum_{i=1}^{n}\left[\E_{(s,\jntact)\sim (d^{\jntcurpi}, \jntcurpi)}\left[\frac{\tarpi^{i}}{\curpi^{i}}A^{\jntcurpi}\right]\right]\right|            \\
    \leq & \left|\J{\jnttarpi}-\J{\jntcurpi}-\frac{1}{n}\frac{1}{1-\gamma}\sum_{i=1}^{n}\left[\E_{(s,\jntact)\sim (d^{\jntcurpi}, \jntcurpi)}\left[\frac{\tarpi^{i}}{\curpi^{i}}A^{\jntcurpi}\right]\right]\right| \\
         & + \frac{n-1}{n}\frac{1}{1-\gamma} \left|\sum_{i=1}^{n}\left[\E_{(s,\jntact)\sim (d^{\jntcurpi}, \jntcurpi)}\left[\frac{\tarpi^{i}}{\curpi^{i}}A^{\jntcurpi}\right]\right]\right|                        \\
    \leq & \frac{2\epsilon^{\jntcurpi}}{1-\gamma} \frac{n+1}{n} \sum_{j=1}^{n}\alpha^{j} + \frac{n-1}{n}\sum_{i=1}^{n} \frac{1}{1-\gamma}\alpha^{i} \cdot 2\epsilon^{\jntcurpi}                                    \\
    =    & \frac{4\epsilon^{\jntcurpi}}{1-\gamma} \sum_{i=1}^{n}\alpha^{i}
\end{align*}

\textbf{CoPPO}. We prove the results of CoPPO in a unified and convenient form. For CoPPO, $\Lo{\jntcurpi}{\jnttarpi}=\J{\jntcurpi}+\frac{1}{1-\gamma}\E_{(s, \jntact)\sim (d^{\jntcurpi}, \jnttarpi)}[A^{\jntcurpi}(s,\jntact)]$, we prove the bound using \Cref{lemma:genralized_adv_discrepancy_bound}.

\begin{align*}
         & \left|\J{\jnttarpi} - \J{\jntcurpi} - \frac{1}{1-\gamma}\E_{(s, \jntact)\sim (d^{\jntcurpi}, \jnttarpi)}[A^{\jntcurpi}] \right|                                                                          \\
    \leq & \frac{1}{1-\gamma} \left| \E_{(s, \jntact) \sim (d^{\jnttarpi}, \jnttarpi)}[A^{\jntcurpi}] - \E_{(s, \jntact)\sim (d^{\jntcurpi}, \jnttarpi)}[A^{\jntcurpi}] \right|                                     \\
    \leq & \sum_{t=0}^{\infty} \gamma^{t} \left| \E_{(s, \jntact) \sim (Pr^{\jnttarpi}, \jnttarpi)}\left[A^{\jntcurpi}\right] -\E_{(s, \jntact) \sim (Pr^{\jntcurpi}, \jnttarpi)}\left[A^{\jntcurpi}\right] \right| \\
    \leq & 4\epsilon^{\jntcurpi} \sum_{t=0}^{\infty} \gamma^{t} \sum_{i=1}^{n}\alpha^{i} \left(1-(1-D_{TV}^{max}(\jntcurpi\|\jnttarpi))^{t}\right)                                                                  \\
    % \leq & 4\epsilon^{\jntcurpi} \sum_{t=0}^{\infty} \gamma^{t} \sum_{i=1}^{n}\alpha^{i} \left(1-(1-\sum_{i=1}^{n}\alpha^{i})^{t}\right)                                                                            \\
    \leq & 4\epsilon^{\jntcurpi} \sum_{i=1}^{n}\alpha^{i} \left(\frac{1}{1-\gamma}-\frac{1}{1-\gamma(1-\sum_{j=1}^{n}\alpha^{j})}\right)
\end{align*}

\textbf{HAPPO}. Following the proof of Lemma 2 in \citet{kuba2022trust}, we know that HAPPO has the same monotonic improvement bound as that of CoPPO. For the monotonic improvement of a single agent, we formulate the surrogate objective of agent $i$ using HAPPO as $\J{\jntinterpi{i-1}}+\frac{1}{1-\gamma}\E_{(s,\jntact) \sim (d^{\jntcurpi}, \jntinterpi{i})}[A^{\jntcurpi}(s,\jntact)] - \frac{1}{1-\gamma}\E_{(s,\jntact) \sim (d^{\jntcurpi}, \jntinterpi{i-1})}[A^{\jntcurpi}(s,\jntact)]$, as shown in Proposition 3 of \citet{kuba2022trust}. Following the proof of \Cref{thm:single_agent_bound}, we get the following inequality.

% explain why HAPPO does not have monotonic improvement on each agent.

\begin{align*}
    & \left|\J{\jntinterpi{i}}-\J{\jntinterpi{i-1}}-\frac{1}{1-\gamma}\E_{(s,\jntact)\sim (d^{\jntcurpi}, \jntinterpi{i})}\left[A^{\jntcurpi}\right] + \frac{1}{1-\gamma}\E_{(s,\jntact) \sim (d^{\jntcurpi}, \jntinterpi{i-1})}[A^{\jntcurpi}(s,\jntact)]\right|                                                                                                             \\
\leq   & \frac{1}{1-\gamma}\left| \E_{(s,\jntact)\sim (d^{\jntinterpi{i}}, \jntinterpi{i})}\left[A^{\jntinterpi{i-1}}\right] -\E_{(s,\jntact)\sim (d^{\jntcurpi}, \jntinterpi{i})}\left[A^{\jntcurpi}\right] \right| + \frac{1}{1-\gamma}\left|\E_{(s,\jntact) \sim (d^{\jntcurpi}, \jntinterpi{i-1})}[A^{\jntcurpi}(s,\jntact)]\right|                                                       \\
\leq    & \frac{1}{1-\gamma}\left| \E_{(s,\jntact)\sim (d^{\jntinterpi{i}}, \jntinterpi{i})}\left[A^{\jntinterpi{i-1}}\right] - \frac{1}{1-\gamma}\E_{(s,\jntact)\sim (d^{\jntcurpi}, \jntinterpi{i})}\left[A^{\jntinterpi{i-1}}\right]\right|                                                                  \\
    & + \frac{1}{1-\gamma}\left|\E_{(s,\jntact)\sim (d^{\jntcurpi}, \jntinterpi{i})}\left[A^{\jntinterpi{i-1}}\right] -\E_{(s,\jntact)\sim (d^{\jntcurpi}, \jntinterpi{i})}\left[A^{\jntcurpi}\right] \right| + 2\frac{1}{1-\gamma}\sum_{j\in e^{i}}\alpha^{j} e^{\jntcurpi}                                                          \\
\leq & 4 \epsilon^{\jntinterpi{i-1}} \alpha^{i}\sum_{t=0}^{\infty} \gamma^{t} (1-(1-\sum_{j \in (e^{i} \cup \{i\})}\alpha^{j})^{t})  \\
& + \frac{1}{1-\gamma}\E_{(s,\jntact)\sim (d^{\jntcurpi}, \jntinterpi{i})}\left[\left|A^{\jntinterpi{i-1}}-A^{\jntcurpi}\right|\right] + 2\frac{1}{1-\gamma}\sum_{j\in e^{i}}\alpha^{j} e^{\jntcurpi} \\
\leq & 4 \epsilon^{\jntinterpi{i-1}} \alpha^{i} (\frac{1}{1-\gamma} - \frac{1}{1-\gamma(1-\sum_{j \in (e^{i} \cup \{i\})}\alpha^{j})})   + \frac{1}{1-\gamma} \left[4\alpha^{i}\epsilon^{\jntinterpi{i-1}}+4\sum_{j \in e^{i}}\alpha^{j}\epsilon^{\jntcurpi}\right] \\
\end{align*}

The right side of the last inequality is not a monotonic improvement bound, or it does not provide a guarantee for improving the expected performance $\J{\jntinterpi{i}}$ since the term $\sum_{j\in e^{i}}\alpha^{j}\epsilon^{\jntcurpi}$ is not controllable for agent $i$, whether through policy improvement or value learning. The uncontrollable term means the expected performance may not be improved even if the total variation distances of consecutive policies are well constrained.

\subsection{Comparisons on Monotonic Improvement Bounds}

CoPPO and HAPPO have the same monotonic bound that is tighter than that of MAPPO. A2PO achieves the tightest monotonic bound given mild assumptions about the errors of \trace{}, which is valid and easy to achieve since \trace{} is a contraction operator.
A sufficient condition that A2PO has the tightest bound is that $\xi^{i} < \frac{\gamma(1-\gamma)\sum_{j \in \mathcal{N} - e^{i} - \{i\}} \alpha^{j}}{(1-\gamma(1-\sum_{j \in e^{i}\cup \{i\}}\alpha^{j}))(1-\gamma(1-\sum_{j=1}^{n}\alpha^{j}))}$, for all $i \in \mathcal{N}$.

\subsection{Preceding-agent Off-policy Correction}
\label{app:preopc}

In Retrace($\lambda$) \citep{Munos2016}, consider the current policy as $\jntinterpi{i=1}$ and base policy as $\jntcurpi$, we have the following definition:

\begin{equation*}
    \mathcal{R}_{t} = r_{t} + \gamma Q_{t+1} + \sum_{k\geq 1}\gamma^{k}
    \big(
    \prod_{j=1}^{k}\lambda\min\big(1.0, \frac{\jntinterpi{i-1}(\jntact_{t+j}|s_{t+j})}{\jntcurpi(\jntact_{t+j}|s_{t+j})}\big)
    \big)
    (r_{t+k}+\gamma Q_{t+k+1} - Q_{t+k})~,
\end{equation*}

Following that same structure, we have: 

\begin{equation*}
    \mathcal{R}_{t} = r_{t} + \gamma V_{t+1} + \sum_{k\geq 1}\gamma^{k}
    \big(
    \prod_{j=1}^{k}\lambda\min\big(1.0, \frac{\jntinterpi{i-1}(\jntact_{t+j}|s_{t+j})}{\jntcurpi(\jntact_{t+j}|s_{t+j})}\big)
    \big)
    (r_{t+k}+\gamma V_{t+k+1} - V_{t+k})~,
\end{equation*}

By subtracting $V_{t}$, we get the definition of PreOPC. Or one can get $\gamma A^{\jntcurpi, \jntinterpi{i-1}}$ by substituting $r_{t}+\gamma V_{t+1}$ for $Q_{t}$ and subtracting $r_{t}+\gamma V_{t+1}$.

\subsection{Why Off-policyness is More Serious in Sequential Update Scheme?}

As shown in \Cref{fig:app_more_abl_trace}, the off policy correction in sequential update algorithms improves the performance significantly while similar performance gaps are not observed when used in simultaneous update algorithms. We attribute the difference to the influence of the clipping mechanism on the total variation distance.

From \Cref{cor:2}, $D_{TV}(\jntcurpi\|\jnttarpi) < \sum_{i=1}^{n} D_{TV}(\curpi^{i}\|\tarpi^{i})$. Although we can not prove exact relations, clipping the agents independently tends to larger total variation distances between the current and future policies of the agents, leading to more `off-policyness' in sequential update algorithms.

%% file: sections/appendix-experiments.tex
\section{Experimental Details}
\label{app:exp}

\subsection{Implementation}
\label{app:implementation}
For a fair comparison, we (re)implement \alg{} and the baselines based on the implementation of MAPPO. 
% \footnote{https://github.com/marlbenchmark/on-policy.}
We keep the same structures for all the algorithms and tune all the algorithms following the same process, i.e., a grid search over a small collection of hyper-parameters, to avoid the influence of different implementation details on the results. The grid search is performed on three hyper-parameters: the learning rate, $\lambda$ and the agent block num in the tasks with numerous agents.

The algorithms, including \alg{} and baselines, are implemented into both parameter sharing and parameter independent versions. \alg{} in the parameter sharing version is implemented as in \Cref{alg:a2poframework_ps}. The main modifications are colored in blue. We rearrange the loops of agents and ppo epochs. The number of ppo epochs is divided by $n$ for comparable updating times with the simultaneous algorithms. The approximated advantage is estimated by correcting the action probabilities of all the agents given such $e^{i}$.

\begin{algorithm}[htbp]
    \caption{Agent-by-agent Policy Optimization (Parameter Sharing)}
    \label{alg:a2poframework_ps}
    Initialize the shared joint policy $\jntcurpi_{0}=\{\curpi_{0}^{1}, \ldots, \curpi_{0}^{n}\}$ with $\curpi^{1}_{0}=\cdots=\curpi^{n}_{0}$, and the global value function $\V$.\\

    \For{iteration $m=1,2,\ldots$} {
    Collect data using $\jntcurpi_{m-1}$.\\

    \textcolor{blue}{Policy $\jntcurpi_{m}=\jntcurpi_{m-1}$}.\\

    \For{\textcolor{blue}{$\lceil \frac{P}{n} \rceil$ epochs}} {
    \For{\textcolor{blue}{$k=1,\ldots,n$}} {

    Agent $i=\mathcal{R}(k)$, \textcolor{blue}{preceding agents $e^{i}=\{\mathcal{R}(1), \ldots, \mathcal{R}(n-1)\}$}.

    \textcolor{blue}{Joint policy $\jntinterpi{i}=\jntcurpi_{m}$}.\\
    Compute the advantage approximation as $A^{\jntcurpi, \jntinterpi{i-1}}(s, \jntact)$ via \cref{eq:trace}. \label{alg_line:ps_adv}

    Compute the value target $v(s_t) = A^{\jntcurpi, \jntinterpi{i-1}}(s, \jntact) + \V(s)$.

    $\curpi^{i}_{m} = \argmax_{\curpi^{i}_{m}} \tilde{\mathcal{L}}_{\jntinterpi{i-1}}(\jntinterpi{i})$ as in \cref{eq:obj}.

    $\V = \argmin_{\V} \E_{s\sim d^{\jntcurpi}}\|v(s)-V(s)\|^{2}$.\\
    }

    }

    }
\end{algorithm}

% explain how to implement sequential update algos in parameter sharing versions.
% block
Practically, each agent is equipped with a value function, we generate the agent order at once to avoid estimating the advantage function $\frac{n(n-1)}{2}$ times. The order becomes $[1, \ldots, i, \ldots, j, \ldots, n]$ in which $\E|A^{i}|>=\E|A^{j}|$.

\subsection{Experimental Setup and Additional Results}
\label{app:exp_setup_result}

\subsubsection{StarCraftII Multi-agent Challenge}
\label{app:exp_setup_result_smac}

StarCraftII Multi-agent Challenge (SMAC) \citep{Samvelyan2019} provides a wide range of multi-agent tasks in the battle scenarios of StarCraftII. Algorithms adopting parameter sharing have shown superior performance in SMAC, so all the algorithms are implemented as parameter sharing. As shown in \Cref{tab:app_smac}, we evaluate the algorithms in 12 maps of SMAC with various difficulties, in which the baselines can not achieve 100\% win rates easily. We use the results of Qmix in \citet{Yu2022}. The learning curves for episode return are summarized in \Cref{fig:app_smac}.

\begin{table}[htb]
    \centering
    \small
    \caption{Median win rates and standard deviations on SMAC tasks. `w/ PS' means the algorithm is implemented as parameter sharing}
    \begin{tabular}{llllll|l}
        \toprule
        Map            & Difficulty & MAPPO w/ PS         & CoPPO w/ PS          & HAPPO w/ PS         & A2PO w/ PS     & Qmix w/ PS     \\
        \midrule
        MMM            & Easy       & 96.9(0.988)         & 96.9(1.25)           & 95.3(2.48)          & \textbf{100(1.07)}  & 95.3(2.5) \\
        3s\_vs\_5z     & Hard       & \textbf{100(1.17)}  & \textbf{100(2.08)}   & \textbf{100(0.659)} & \textbf{100(0.534)} & 98.4(2.4) \\
        2c\_vs\_64zg   & Hard       & \textbf{98.4(1.74)} & 96.9(0.521)          & 96.9(0.521)         & 96.9(0.659) & 92.2(4.0)     \\
        3s5z           & Hard       & 84.4(4.39)          & 92.2(2.35)           & 92.2(1.74)          & \textbf{98.4(1.04)} & 88.3(2.9) \\
        5m\_vs\_6m     & Hard       & 84.4(2.77)          & 84.4(2.12)           & 87.5(2.51)          & \textbf{90.6(3.06)} & 75.8(3.7) \\
        8m\_vs\_9m     & Hard       & 84.4(2.39)          & 84.4(2.04)           & 96.9(3.78)          & \textbf{100(1.04)} & 92.2(2.0)  \\
        10m\_vs\_11m   & Hard       & 93.8(18.7)          & 96.9(2.6)            & 98.4(2.99)          & \textbf{100(0.521)} & 95.3(1.0) \\
        6h\_vs\_8z     & Super Hard & 87.5(1.53)          & \textbf{90.6(0.765)} & 87.5(1.49)          & \textbf{90.6(1.32)} & 9.4(2.0) \\
        3s5z\_vs\_3s6z & Super Hard & 82.8(19.2)          & 84.4(2.9)            & 37.5(13.2)          & \textbf{93.8(19.8)} & 82.8(5.3) \\
        MMM2           & Super Hard & 90.6(8.89)          & 90.6(6.93)           & 51.6(9.01)          & \textbf{98.4(1.25)} & 87.5(2.6) \\
        27m\_vs\_30m   & Super Hard & 93.8(3.75)          & 93.8(2.2)            & 90.6(4.77)          & \textbf{100(1.55)} & 39.1(9.8) \\
        corridor       & Super Hard & 96.9(0)             & \textbf{100(0.659)}  & 96.9(0.96)          & \textbf{100(0)}  & 84.4(2.5)   \\
        \midrule
        overall        & /          & 91.1(5.46)          & 92.6(2.2)            & 85.9(3.68)          & \textbf{97.4(2.65)} & 78.4(3.6)\\
        \bottomrule
    \end{tabular}
    % \vspace{-5pt}
    \label{tab:app_smac}
\end{table}

\begin{figure}[htbp]
    \centering
    \includegraphics[width=0.75\linewidth]{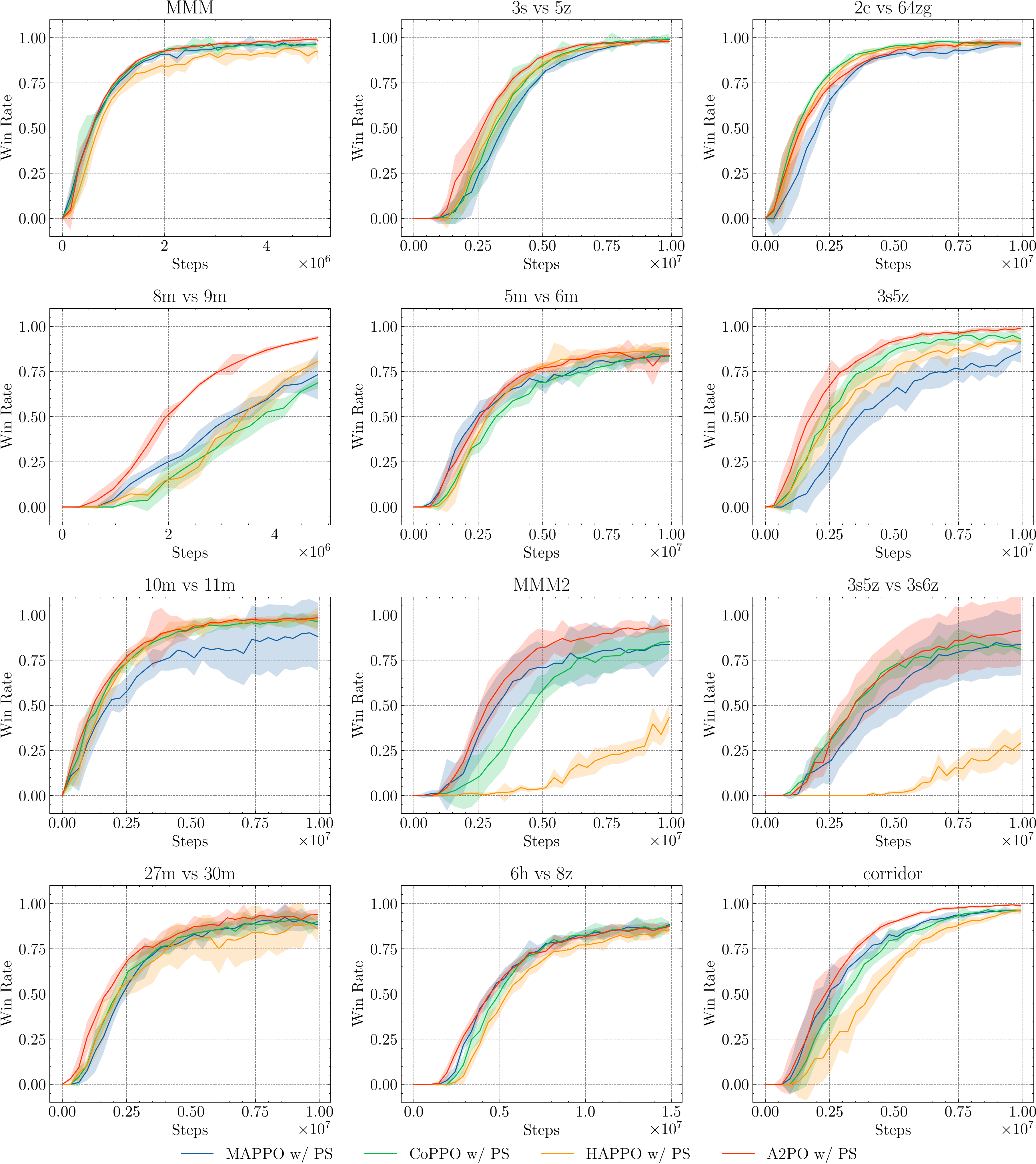}
    \caption{Comparisons of median win rate on SMAC.}
    \label{fig:app_smac}
\end{figure}
\subsubsection{Multi-agent MuJoCo}
\label{app:exp_setup_result_mujoco}

Multi-agent MuJoCo (MA MuJoCo) \citep{Peng2021} contains a range of multi-agent robot continuous control tasks, in which an agent controls the composition of robot joints. MA MuJoCo extends the high-dimensional single-agent locomotion tasks in MuJoCo \citep{Todorov2012}, a widely adopted benchmark for SARL algorithms \citep{Haarnoja2018,He2020}, into the multi-agent case. Agents must cooperate in their actions for robot locomotion, and different agents control different compositions of the robot joints.
We use the reward settings of the original paper but set the environment to be fully observable\footnote{Empirically, we find the fully observable setting does not make the tasks easier because of the information redundancy.}. The agents are heterogeneous and mostly asymmetric in MA-MuJoCo, so we implement the algorithms as parameter-independent. We test 14 tasks of 6 scenarios in MA MuJoCo, as illustrated in \Cref{fig:app_exp_mujoco}.

\begin{figure}[htbp]
    \centering
    \includegraphics[width=0.75\linewidth]{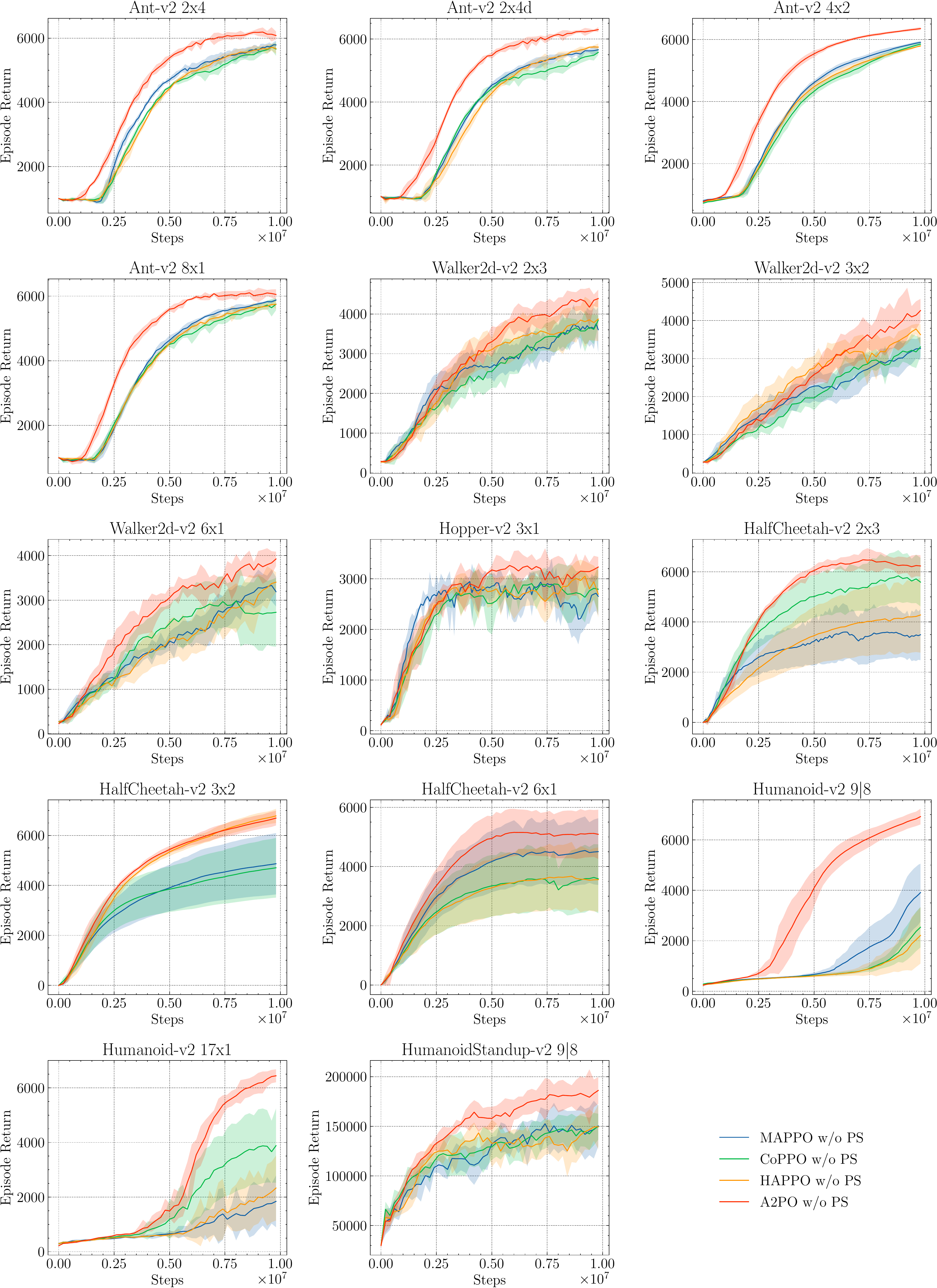}
    \caption{Comparisons of average episode return on MA-MuJoCo.}
    \label{fig:app_exp_mujoco}
\end{figure}

\subsubsection{Multi-agent Particle Environment}
\label{app:exp_setup_result_mpe}
We consider the Navigation task of the Multi-agent Particle Environment (MPE) \citep{Lowe2017} implemented in PettingZoo \citep{Terry2021} which implements MPE with minor fixes and provides convenience for customizing the number of agents and landmarks, and customizing the global and local rewards., with 3 and 5 agents and corresponding numbers of landmarks. The agents are rewarded based on the minimum distance to the landmarks and penalized for colliding with each other, meaning that the reward is entirely up to the coordination behavior. We adopted two different reward settings: Fully Cooperative and General-sum. In the Fully Cooperative setting, the agents share the same reward, while in the General-sum setting, the agents are additionally rewarded based on the local collision detection. The results in \Cref{fig:pettingzoo} show that \alg{} generally outperforms the baselines on the average return and the sample efficiency. Noted that \alg{} is developed in fully cooperative games, the results in the General-sum setting reveal the potential of extending \alg{} into general-sum games. Further, the performance gap between \alg{} and the baselines enlarges with the increasing number of agents.

\begin{figure}[htb]
    \centering
    \includegraphics[width=\linewidth]{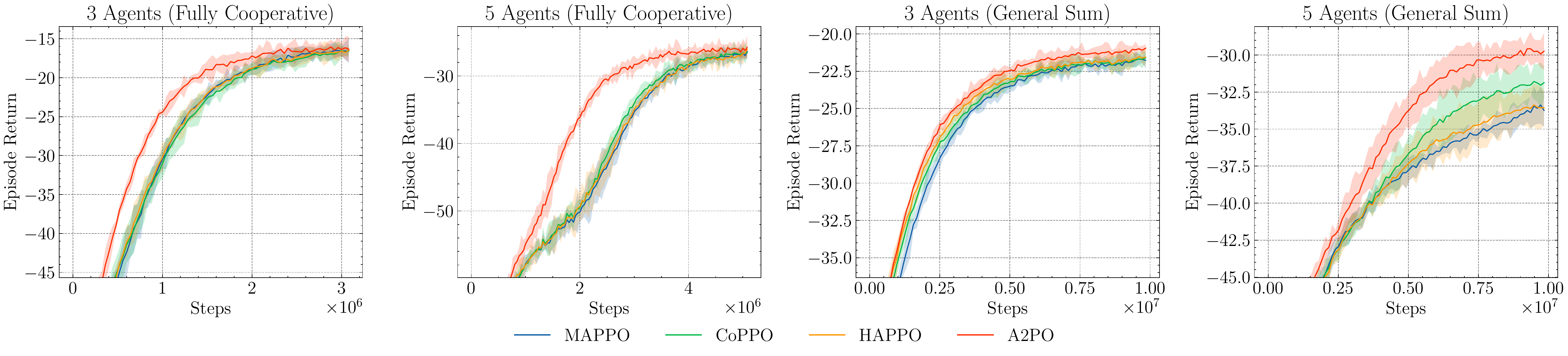}
    \caption{Comparisons of averaged return on the Multi-agent Particle Environment Navigation task. \textbf{Left}: The fully cooperative setting. \textbf{Right}: The general-sum setting.}
    \label{fig:pettingzoo}
\end{figure}

\subsubsection{Google Research Football}
\label{app:exp_setup_result_grf}

\begin{wrapfigure}{r}{0.25\linewidth}
    \centering
    \vspace{-10pt}
    \includegraphics[width=\linewidth]{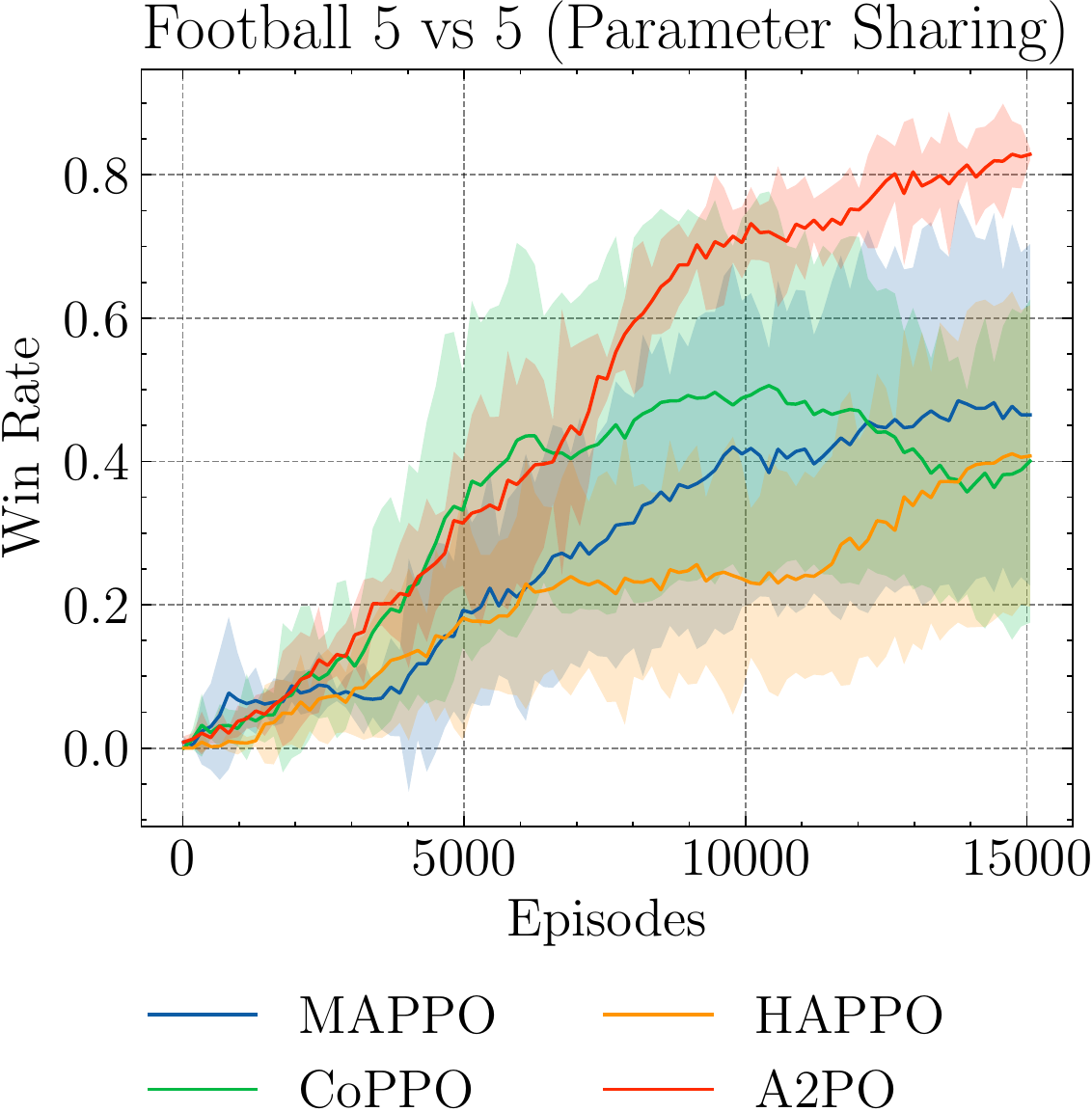}
    \caption{5-vs-5 scenario with Parameter sharing.}
    \label{fig:app_grf_ps}
    \vspace{-10pt}
\end{wrapfigure}
In the above experiments, we have evaluated \alg{} in tasks where agents can learn both their micro-operations and coordination behaviors (SMAC and MA-MuJoCo) and tasks where agents can only learn coordination behaviors (the Navigation task). However, the coordination behaviors in the above tasks are relatively easy to discover, e.g., agents learn to concentrate their fire to shoot the enemies and cover each other in SMAC.
Recent works \citep{Wen2022,Yu2022} have conducted experiments on Google Research Football academic scenarios with a small number of players and easily accessible targets, making the coordination behavior also easy to discover.
In contrast, we evaluate \alg{} in the full-game scenarios, where the players of the left team, except for the goalkeeper, are controlled to play a football match against the right team controlled by the built-in AI provided by GRF. The agents in the full-game scenarios have high-dimensional observations, complex action spaces, and a long-span timescale (3000 steps).
We reconstruct the observation space and design a dense reward to facilitate training in these scenarios based on Football-paris.
% \url{https://github.com/seungeunrho/football-paris}. 
The observation is formed to be agent-specific. The reward function estimates the behaviors of the entire team, including scoring, and carrying the ball to the opponent's restricted area et al., but not the individual behaviors such as ball-passing \citep{Li2021}.
We implement all the algorithms for the 5-vs-5 scenario as both parameter sharing and parameter-independent. The additional results with algorithms implemented as parameter sharing are shown in \Cref{fig:app_grf_ps}, in which \alg{} gets free from the trouble that the controlled agents have similar behavior and compete for the ball \citep{Li2021}.

We implement all the algorithms on the 11-vs-11 scenario as parameter sharing using MALib \citep{Zhou2021} for acceleration and train the algorithms for 300M environment steps. We summarize the learned behaviors observed in the game videos:
\begin{itemize}
    \item \textbf{Basic Skills}. The agents trained by MAPPO and CoPPO perform unsatisfactorily in basic skills such as dribbling, shooting, and the agents even run out of bounds frequently. In contrast, the agents trained by HAPPO and \alg{} perform better in the basic skills. We attribute the problems to the non-stationarity issue that seriously influences the simultaneous updating algorithms. We also note that the agents trained by all the algorithms fail to understand the off-side mechanism and occasionally gather together on the opponent's bottom line.
    \item \textbf{Passing and Receiving Coordination}. We analyze the direct way for coordination: passing and receiving the ball. As illustrated in \Cref{tab:grf_coordination}, the agents trained by MAPPO have the lowest number of successful passes and the lowest successful pass rate, and we can hardly observe the agents passing the ball. Agents trained by CoPPO perform better on passing the ball but suffer from poor basic skills, and get tackled after receiving the ball. Agents trained by HAPPO prefer passing the ball without considering the teammates' situations, e.g., the receiver is marked by several opponents. Agents trained by \alg{} can pass the ball to their teammates in a way that leads to a score. We attribute the performance gain to the \trace{}, which means that agents estimate the teammates' situations and intentions better.
\end{itemize}

\begin{figure}[htb]
    \centering
    \vspace{-5pt}
    \includegraphics[width=\linewidth]{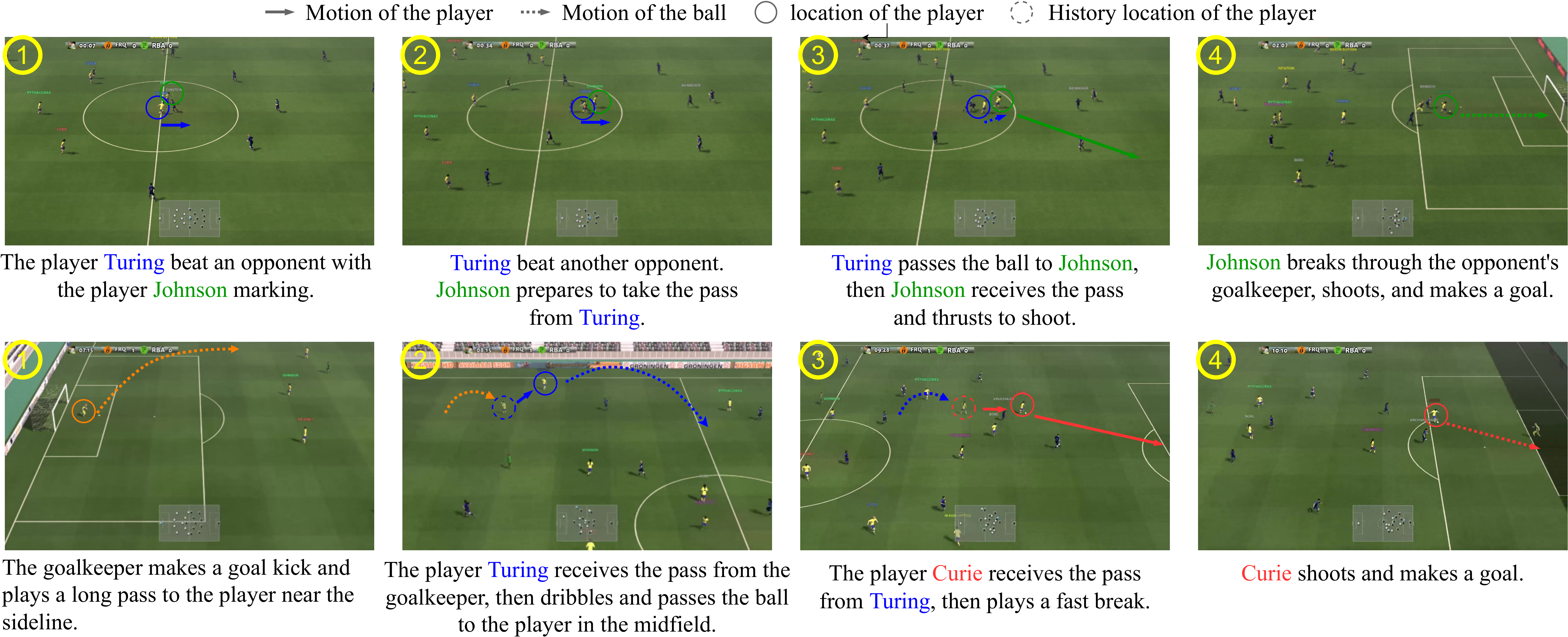}
    \caption{Visualization of trained \alg{} policies on the Google Research Football 11-vs-11 scenario, which shows that \alg{} encourages complex cooperation behaviors to make a goal. \textbf{Top}: Player Turing and Johnson cooperate to beat multiple opponents to break through the defense and make a goal. 
    \textbf{Bottom}: The goalkeeper, player Turing, and Curie achieve the pass and receive cooperation twice. A fast thrust is made by consecutively passing the ball.}
    \vspace{-5pt}
    \label{fig:case_grf}
\end{figure}

We further visualize the learned behaviors of \alg{} in \Cref{fig:case_grf}. In the top of \Cref{fig:case_grf}, two players cooperatively break through the opponent's defense and complete a passing and receiving coordination for scoring. In the bottom of \Cref{fig:case_grf}, three players make a fast thrust by two long passes: the goalkeeper passes the ball to the player at the edge, and the player at the edge passes the ball to the player behind the opponents. The complex coordination strategies are hardly observed in other baselines.

\subsubsection{Ablation}
\label{app:add_exp_ablation}

\textbf{Preceding agent off-policy correction}. More ablations on \trace{} are shown in \Cref{fig:app_more_abl_trace}. The baselines are:
\begin{itemize}
    \item MAPPO w/ V-trace, CoPPO w/ V-trace: Simultaneous update methods with advantage estimation as V-trace.
    \item HAPPO w/ PreOPC: HAPPO with advantage estimation as PreOPC.
\end{itemize}

In this ablation study, the baselines are equipped with off-policy correction methods. The experiment yields the following three conclusions:
\begin{itemize}
    \item The results firstly support the conclusion in \Cref{sec:preopc} that applying PreOPC to sequential update methods results in a greater performance improvement than applying V-trace to simultaneous update methods.
    \item Secondly, the primary distinction between A2PO and HAPPO with PreOPC is the clipping objective. The results demonstrate that the clipping objective derived from the single-agent improvement bound contributes to the performance improvement. 
    \item And thirdly, although we were unable to assess the error of PreOPC, we compare A2PO with RPISA-PPO, which can be viewed as A2PO algorithms with error-free off-policy correction methods (the advantage estimation is error-free) at the expense of sample inefficiency.  \alg{} reaches or outperforms the asymptotic performance of RPISA-PPO. \alg{} outperforms \seqrpippo{} since \seqrpippo{} suffers from performance degradation as a result of agents updating policies with separated data \citep{Taheri2022}.
\end{itemize}

We further analyze the sensitivity to the hyper-parameter $\lambda$. Results in \Cref{fig:app_more_sen} illustrate that \trace{} does not introduce more sensitivity.

\begin{figure}[htbp]
    \centering
    \includegraphics[width=\linewidth]{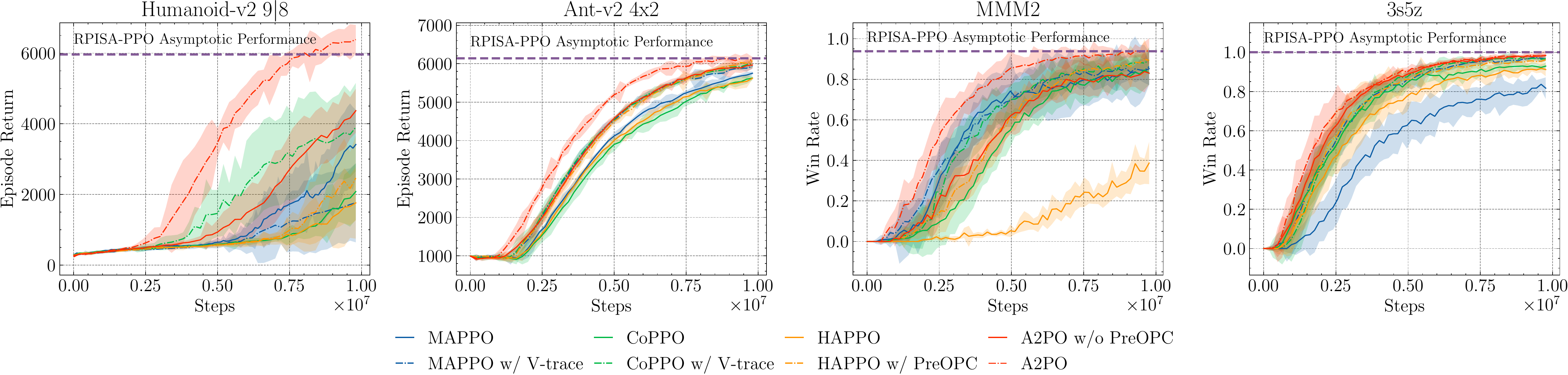}
    \caption{Ablation experiments on \trace{}.}
    \label{fig:app_more_abl_trace}
\end{figure}

\begin{figure}[htbp]
    \centering
    \includegraphics[width=\linewidth]{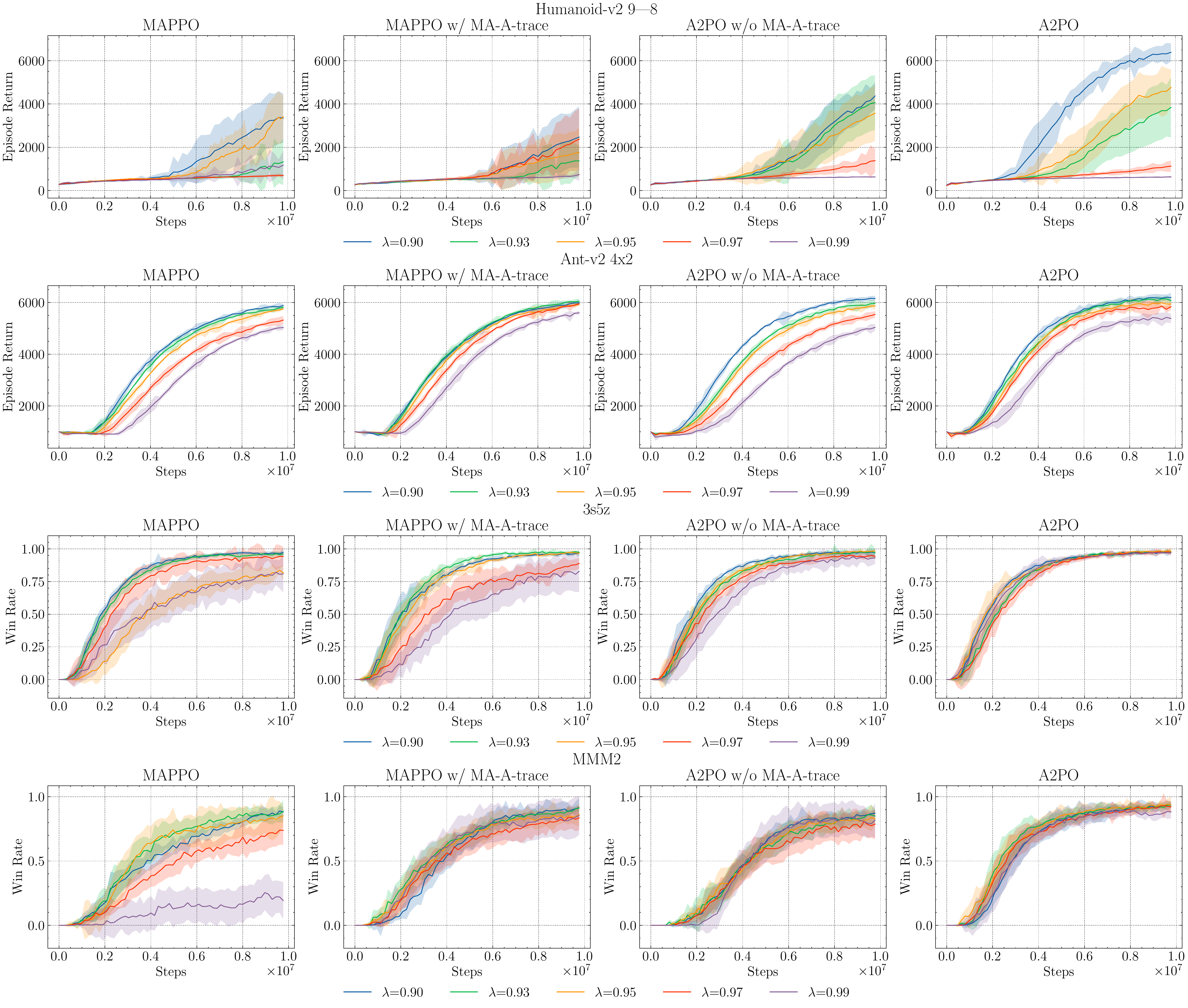}
    \caption{Sensitivity analysis of $\lambda$.}
    \label{fig:app_more_sen}
\end{figure}

\textbf{Agent Selection Rule}. More ablations on the agent selection rules are shown in \Cref{fig:app_more_abl_trace}. We compare two additional rules: `Reverse-greedy' and `Reverse-semi-greedy'. `Reverse' means selecting the agent with the minimal advantage first. While we observe that the effect of the selection rule becomes less significant in tasks with homogeneous or symmetric agents.

Going deeper into the effects of agent selection rules, we show that the agents with implicit guidance from the advantage estimation benefit from greedily selecting agents in \Cref{fig:fair,fig:more_seq_ana}. More even bars appear in one fig means the agents are more balanced in terms of the guidance from the advantage estimation. Take the agent 10 in \Cref{fig:fair} for example, under `Cyclic' and `Random' rules, agent 10 perform the worst with high proportions, while it has higher proportions in prior ranks under `Greedy' rule. 

\begin{figure}
    \centering
    \includegraphics[width=\linewidth]{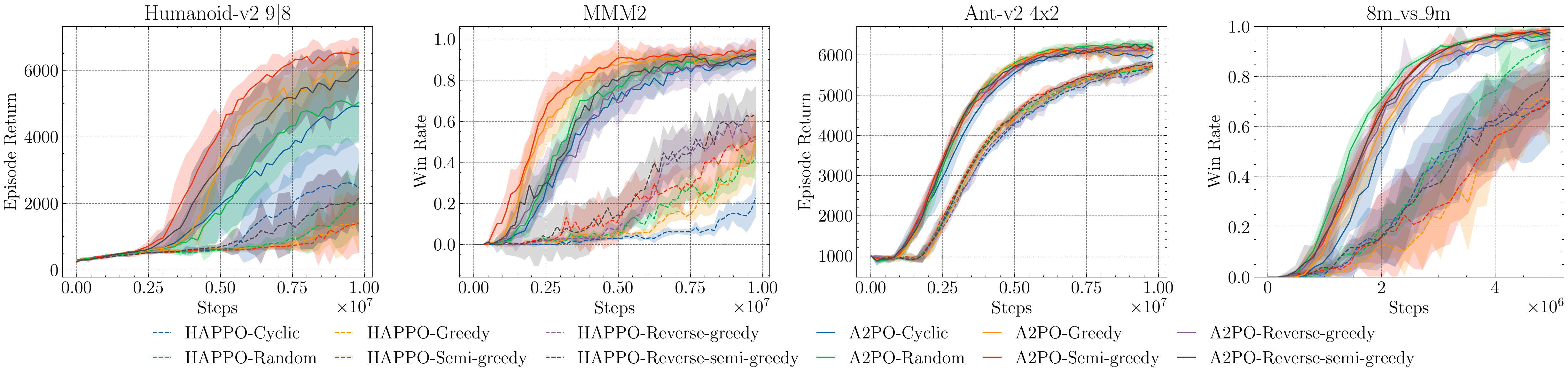}
    \caption{Ablation experiments on the agent selection rules. \textbf{Left}: Heterogeneous or asymmetric agents. \textbf{Right}: Homogeneous or symmetric agents.}
    \label{fig:app_more_abl_seq}
\end{figure}

\begin{figure}[htbp]
    \centering
    \begin{subfigure}[t]{0.72\linewidth}
        \centering
        \includegraphics[height=0.31\linewidth]{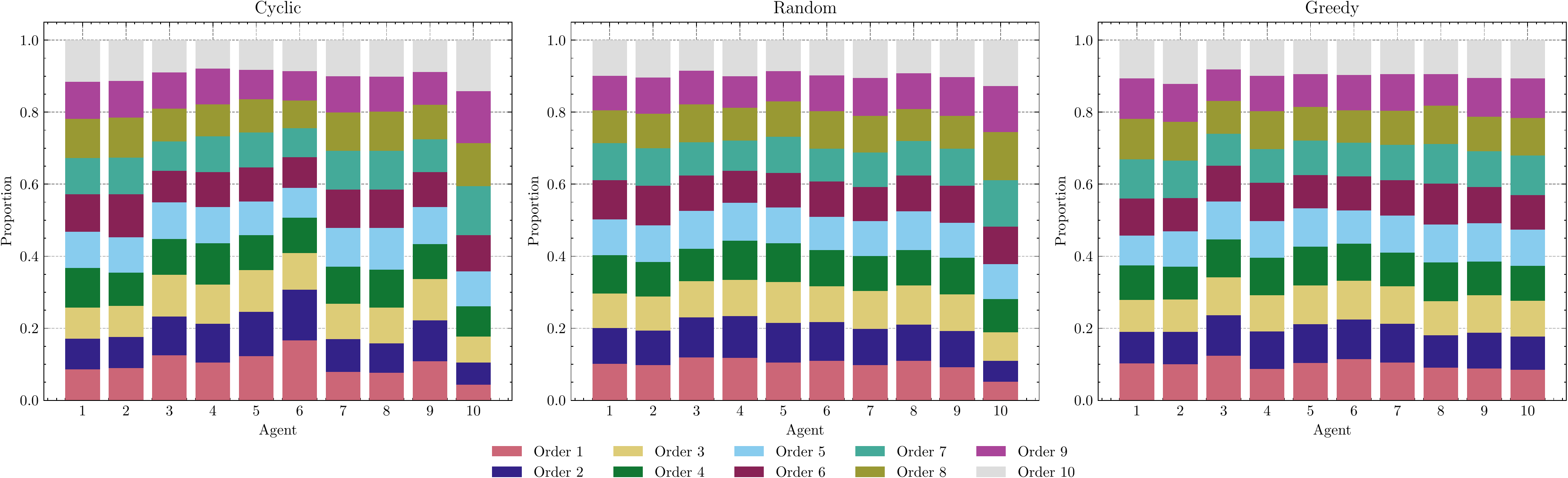}
        \caption{The imbalance of the agents. The bar of agent $i$ illustrates the proportion of its ranks in terms of $\E_{s,a^{i}}[|A^{\jntcurpi, \jntinterpi{i}}|]$. Especially, agent $10$ has implicit guidance, i.e., a small absolute value of advantage function when using Cyclic and Random selection rule, but is comparable with other agents with Greedy selection rule.}
        \label{fig:seq_ana_1}
    \end{subfigure}
    \hfill
    \begin{subfigure}[t]{0.24\linewidth}
        \centering
        \includegraphics[height=0.93\linewidth]{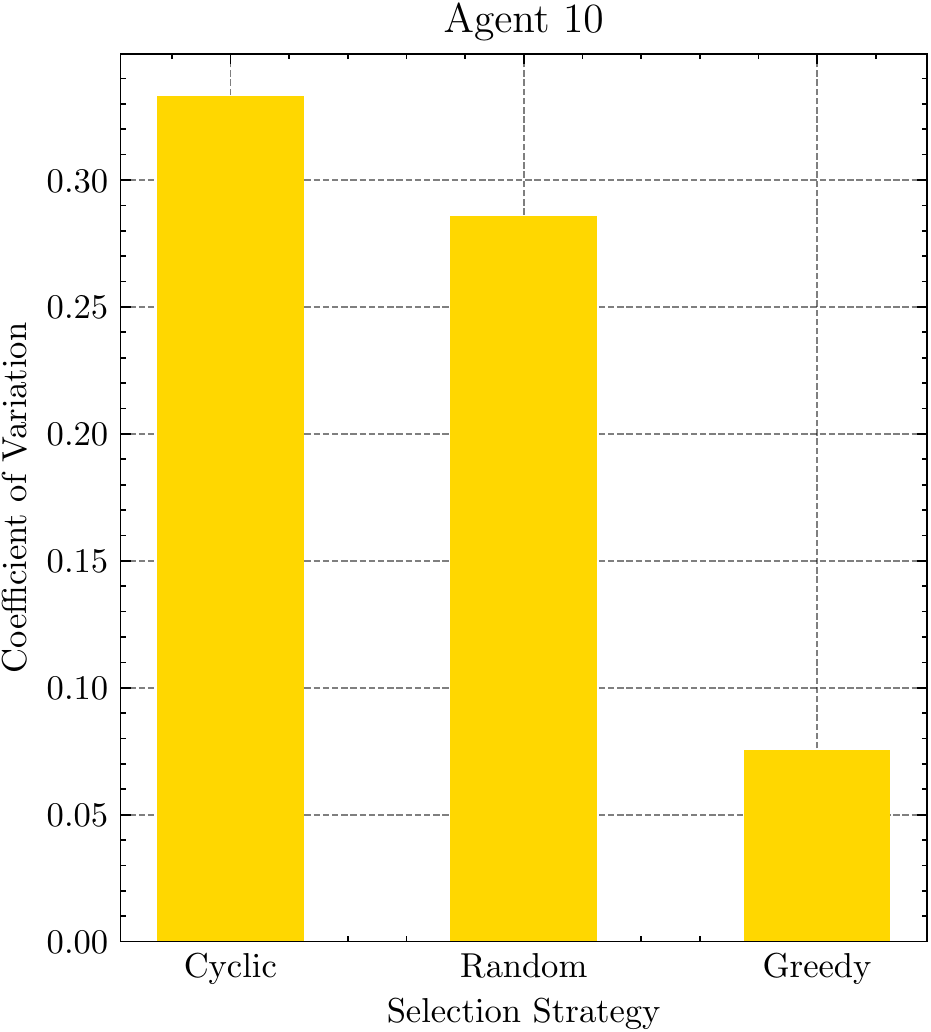}
        \caption{The coefficient of variance of agent $10$'s order proportions.}
        \label{fig:seq_ana_2}
    \end{subfigure}
    \caption{Agents' imbalance in terms of the estimated advantage. The experiment is conducted on the MMM2 task of SMAC.}
    \label{fig:fair}
\end{figure}

\begin{figure}
    \centering
    \includegraphics[width=\linewidth]{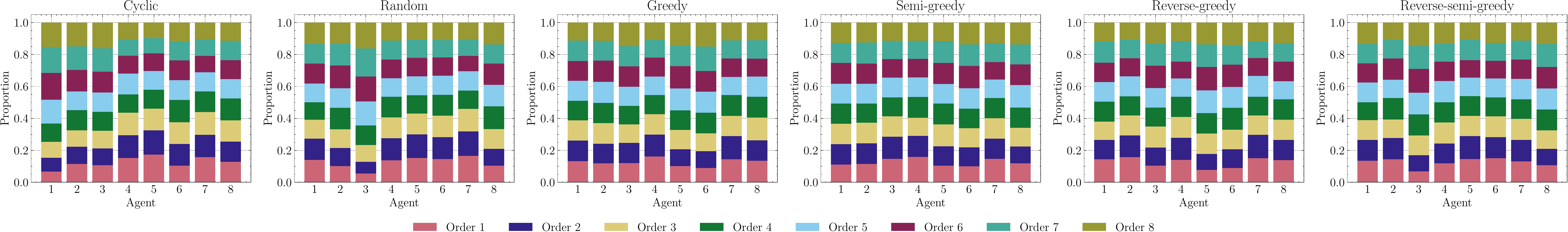}
    \includegraphics[width=\linewidth]{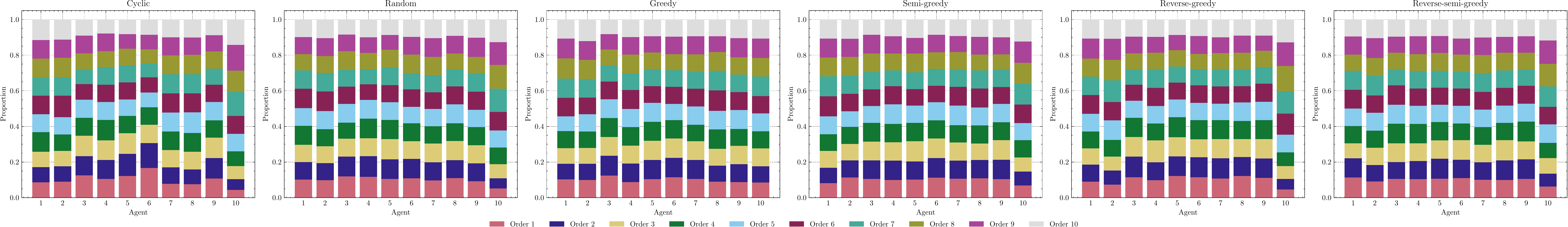}
    \caption{More experiments on the agents' imbalance in terms of the estimated advantage. \textbf{Top}: 3s5z task. \textbf{Bottom}: MMM2 task.}
    \label{fig:more_seq_ana}
\end{figure}

\textbf{Adaptive Clipping Parameter}. More ablations on the adaptive clipping parameter are shown in \Cref{fig:more_adapt}. Similarly, we observe that the effect of the adaptive clipping parameter becomes less significant in tasks with homogeneous or symmetric agents.

\begin{figure}
    \centering
    \includegraphics[width=\linewidth]{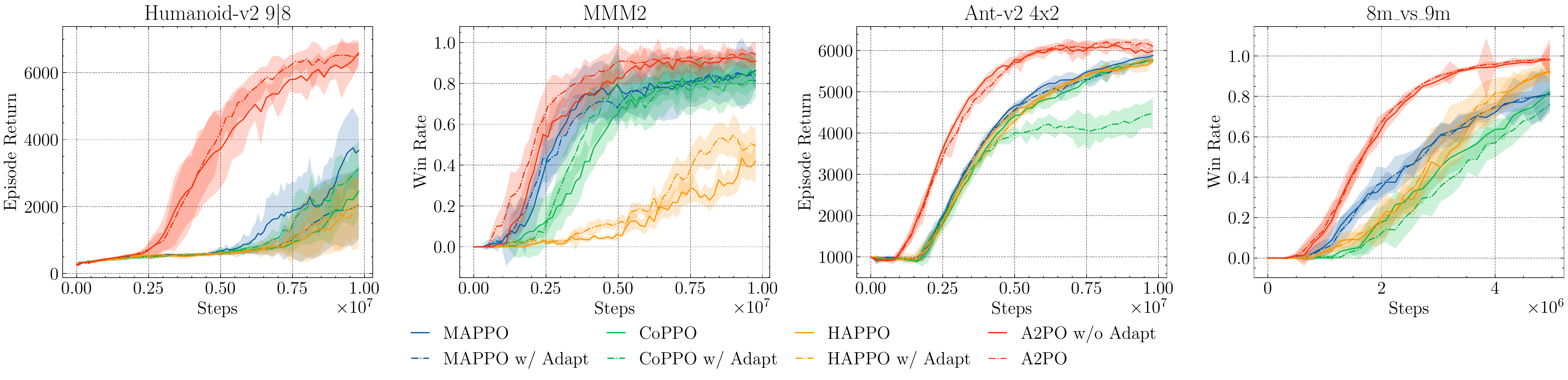}
    \caption{More ablation experiments on the adaptive clipping parameter. \textbf{Left}: Heterogeneous or asymmetric agents. \textbf{Right}: Homogeneous or symmetric agents.}
    \label{fig:more_adapt}
\end{figure}

\subsection{Wall Time Analysis}

Multiple updates in a stage may increase training time, and the need for more training time may impact the scalability of A2PO, which is a common concern regarding the sequential update scheme. Nevertheless, a sequential update scheme will increase training time less than might be expected. Before proceeding, we note that the majority of experiments in our work are synchronously implemented, and the training time consists of the time spent updating policies and collecting samples.

We have proposed a simple yet effective method for controlling training time in order to reduce training time. As a trade-off between performance and training time, we divide the agents into blocks to reduce the number of update iterations. For example, the tasks with 10 agents can be divided into 3 blocks, with sizes 3, 3, 4, respectively, and only 3 updates will be performed in a policy update iteration. From the implementation perspective, since the number of samples used in a single update decreases, the sequential update scheme requires less memory and less updating time when update policies. Therefore, it is possible to control the training time as less than 1.5 times the training time of the simultaneous update methods. In addition, assuming a good implementation, fewer update iterations will be performed if mini-batches are used in a single policy update, as the size of a mini-batch can be greater in sequential update methods under limited memory resources. In such a case, fewer mini-batches will be used, further decreasing the training time. Moreover, sampling consumes the majority of the training time, and the increased updating time appears less significant when analysing the wall time for on-policy algorithms with synchronized implementations.

The training time is depicted in \Cref{tab:time}. A2PO achieves significantly greater performance with only marginally more training time.
In addition, we illustrate the Humanoid $9|8$ comparisons regarding environment steps and training time in \Cref{fig:time_1}, and the comparisons on the GRF 11-vs-11 scenario in \Cref{fig:time_2}. A2PO maintains an advantage in terms of training time.

\begin{figure}
    \centering
    \begin{subfigure}[t]{0.64\linewidth}
    \centering
    \includegraphics[height=0.38\linewidth]{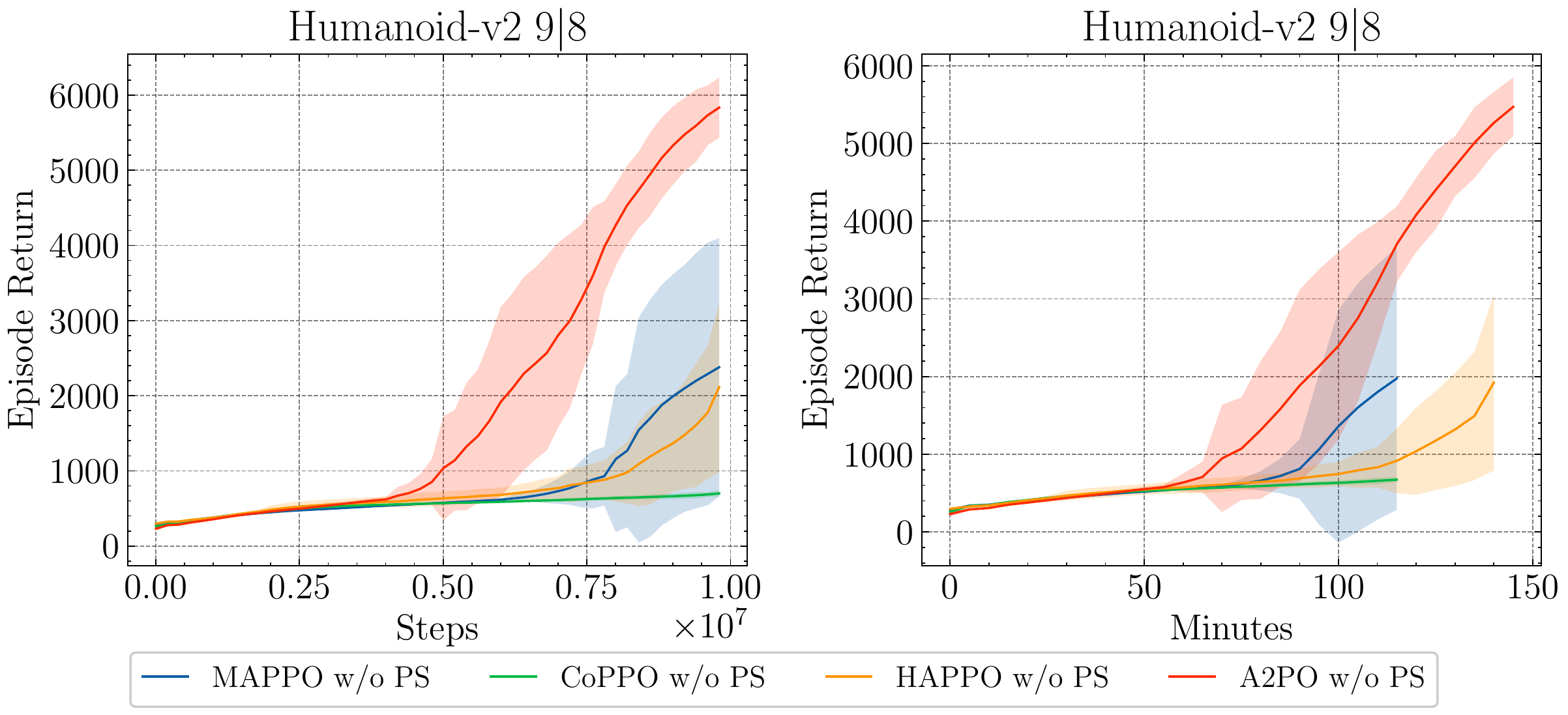}
    \caption{Comparison on Humanoid $9|8$ over both environment steps and training time.}
    \label{fig:time_1}
    \end{subfigure}
    \hfill
    \begin{subfigure}[t]{0.32\linewidth}
    \centering
    \includegraphics[height=0.76\linewidth]{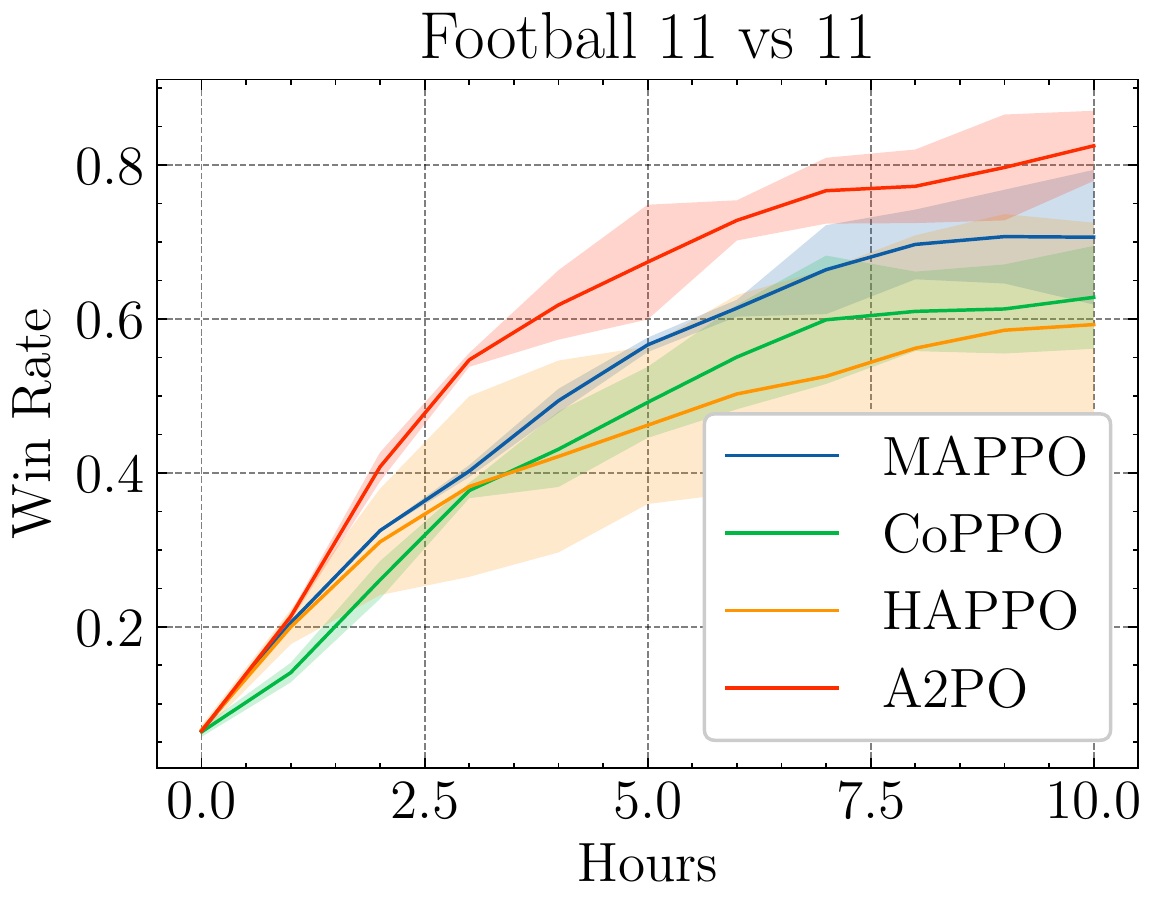}
    \caption{Comparison on GRF 11-vs-11 scenario.}
    \label{fig:time_2}
    \end{subfigure}
    \caption{Wall time Analysis.}
\end{figure}

\begin{table}[htb]
    \centering
    % \tiny
    \caption{The comparison of training duration. The format of the first line in a cell is: Training time(Sampling time+Updating Time). The second line of a cell represents the time normalized.}
    \resizebox{\linewidth}{!}{
    \begin{tabular}{l|l|l|l|l}
        \toprule 
        Task & MAPPO & CoPPO & HAPPO & A2PO \\
        \midrule
        \multirow{2}{*}{3s5z} & 3h29m(3h3m+0h26m) & 3h33m(3h6m+0h27m) & 3h49m(3h7m+0h42m) & 4h32m(3h41m+0h51m) \\
        & 1.00(0.87 + 0.13) & 1.02(0.89 + 0.13) & 1.10(0.89 + 0.20) & 1.30(1.06 + 0.25) \\
        \midrule
        \multirow{2}{*}{27m vs 30m} & 13h23m(8h31m + 4h52m)  & 13h19m(8h24m + 4h55m) & 16h2m(8h20m + 7h42m) & 15h53m(8h7m + 7h46m) \\
        & 1.00(0.64 + 0.36) & 1.00(0.63 + 0.37) & 1.20(0.62 + 0.58) & 1.19(0.61 + 0.58)  \\
        \midrule
        \multirow{2}{*}{Humanoid 9$|$8} & 2h0m(1h45m + 0h15m) & 1h58m(1h43m + 0h15m) & 2h15m(1h45m + 0h30m) & 2h31m(2h0m + 0h31m) \\
        & 1.00(0.87 + 0.13) & 0.99(0.86 + 0.13) & 1.12(0.87 + 0.25) & 1.26(1.00 + 0.26)  \\
        \midrule
        \multirow{2}{*}{Ant 4x2} & 6h42m(6h16m + 0h26m) & 6h45m(6h19m + 0h26m) & 7h29m(6h5m + 1h24m) & 7h2m(5h34m + 1h28m) \\
        & 1.00(0.93 + 0.07) & 1.01(0.94 + 0.07) & 1.12(0.91 + 0.21) & 1.05(0.83 + 0.22)  \\
        \midrule
        \multirow{2}{*}{Humanoid 17x1} & 12h9m(10h6m + 2h3m) & 17h7m(15h5m + 2h2m) & 16h55m(11h2m + 5h53m) & 19h25m(11h59m + 7h26m) \\
        & 1.00(0.83 + 0.17) & 1.41(1.24 + 0.17) & 1.39(0.91 + 0.48) & 1.60(0.99 + 0.61) \\
        \midrule
        \multirow{2}{*}{Football 5vs5} & 34h46m(32h47m + 1h59m) & 32h46m(30h49m + 1h57m) & 39h26m(31h54m + 7h32m) & 37h26m(30h2m + 7h24m) \\
        & 1.00(0.94 + 0.06) & 0.94(0.89 + 0.06) & 1.13(0.92 + 0.22) & 1.08(0.86 + 0.21) \\
        \bottomrule
    \end{tabular}
    }
    \label{tab:time}
\end{table}

\subsection{Hyper-parameters}
\label{app:hypers}

We tune several hyper-parameters in all the benchmarks, other hyper-parameters refer to the settings used in MAPPO. $c_{\epsilon}$ are selected to be $0.5$ in all the tasks.

\subsubsection{StarCraftII Multi-agent Challenge}

We list the hyper-parameters used for each task of SMAC in \Cref{tab:app_hyper_smac}.

\begin{table}[h]
\centering
\caption{Hyper-parameters in SMAC.}
\label{tab:app_hyper_smac}
\begin{tabular}{@{}lllll@{}}
\toprule
Hyperparameters & agent block & ppo epoch & $\lambda$ & $\epsilon$ \\ \midrule
MMM             & 3            & 12        & 0.95      & 0.2        \\
3s\_vs\_5z      & 3            & 15        & 0.95      & 0.05       \\
2c\_vs\_64zg    & 2            & 5         & 0.95      & 0.2        \\
3s5z            & 3            & 8         & 0.95      & 0.2        \\
5m\_vs\_6m      & 2            & 10        & 0.93      & 0.05       \\
8m\_vs\_9m      & 5            & 15        & 0.95      & 0.05       \\
10m\_vs\_11m    & 2            & 10        & 0.97      & 0.2        \\
6h\_vs\_8z      & 2            & 8         & 0.99      & 0.2        \\
3s5z\_vs\_3s6z  & 2            & 5         & 0.90      & 0.2        \\
MMM2            & 2            & 5         & 0.95      & 0.2        \\
27m\_vs\_30m    & 3            & 5         & 0.95      & 0.2        \\
corridor        & 2            & 5         & 0.95      & 0.2        \\ \bottomrule
\end{tabular}
\end{table}

\subsubsection{Multi-agent MuJoCo}

For the model structure in MA MuJoCo, the output from the last layer is processed by a Tanh layer and the action distribution is modeled as a Gaussian distribution initialized with mean as 0 and log std as -0.5. The probability output of different actions are averaged when computing the policy ratio. The common hyper-parameters used in MA MuJoCo are listed in \Cref{tab:app_hyper_mujoco_1}.

\begin{table}[h]
\centering
\caption{Common hypermeters in MA MuJoCo.}
\label{tab:app_hyper_mujoco_1}
\begin{tabular}{@{}ll@{}}
\toprule
Hyperparameters & Values \\ \midrule
entropy         & 0      \\
gain            & 0.01   \\
batch size      & 4000   \\ \bottomrule
\end{tabular}
\end{table}

\begin{table}[h]
\centering
\small
\caption{Hypermeters for the scenarios in MA MuJoCo.}
\label{tab:app_hyper_mujoco_2}
\begin{tabular}{@{}lllllll@{}}
\toprule
Hyperparameters & Ant & HalfCheetah & Hopper & Humanoid & HumanoidStandup & Walker2d \\ \midrule
agent block & 8x1:4 & /    & /    & 17x1:5 & 17x1:4 & /    \\
ppo epoch   & 8     & 5    & 8    & 5      & 5      & 5    \\
actor lr    & 3e-4  & 3e-4 & 1e-4 & 3e-4   & 3e-4   & 3e-4 \\
critic lr   & 3e-4  & 3e-4 & 1e-4 & 3e-4   & 3e-4   & 3e-4 \\
$\lambda$   & 0.93  & 0.93 & 0.95 & 0.9    & 0.93   & 0.93 \\
$\epsilon$  & 0.2   & 0.2  & 0.1  & 0.2    & 0.2    & 0.2  \\
\bottomrule
\end{tabular}
\end{table}

\subsubsection{Multi-agent Particle Environment}

We list the hyper-parameters used in MPE in \Cref{tab:app_hyper_mpe}.

\begin{table}[h]
\centering
\caption{Hypermeters for the scenarios in MPE.}
\label{tab:app_hyper_mpe}
\begin{tabular}{@{}ll@{}}
\toprule
Hyperparameters & Values \\ \midrule
ppo epoch       & 8      \\
chunk length    & 5      \\
entropy         & 0.05   \\
actor lr        & 2e-4   \\
critic lr       & 2e-4   \\
$\lambda$       & 0.97   \\
$\epsilon$      & 0.2    \\ \bottomrule
\end{tabular}
\end{table}

\subsubsection{Google Research Football}

We list the hyper-parameters used in the GRF 5-vs-5 scenario in \Cref{tab:app_hyper_grf}.

\begin{table}[h]
\centering
\caption{Hypermeters for the scenarios in MPE.}
\label{tab:app_hyper_grf}
\begin{tabular}{@{}ll@{}}
\toprule
Hyperparameters & Values \\ \midrule
ppo epoch       & 10     \\
chunk length    & 10     \\
entropy         & 0.001  \\
actor lr        & 5e-4   \\
critic lr       & 5e-4   \\ 
$\lambda$       & 0.95   \\
$\epsilon$      & 0.25   \\
$\gamma$        & 0.995  \\
\bottomrule
\end{tabular}
\end{table}

%% file: sections/appendix-misc.tex
\section{The Related Work of Other MARL Methods}\label{app:others}

\textbf{Value decomposition methods}. The value decomposition methods such as VDN \citep{Sunehag2017} and Qmix \citep{rashid2018qmix}, factorize the joint value function and adopt the centralized training and decentralized execution paradigm. The Individual-Global-MAX (IGM) principle is proposed to ensure consistency between the joint and local greedy action selections in the joint $Q$-value function $Q_{tot}(\bm{\tau}, \bm{a})$ and the individual $Q$-value function $\{Q^{i}(\tau^{i},a^{i}\}_{i=1}^{n}$: $\forall \bm{\tau} \in \mathcal{T}, \arg\max_{\bm{a}\in \mathcal{A}} Q_{tot}(\bm{\tau}, \bm{a})=(\arg\max_{a^1 \in \mathcal{A}^{1}} Q^{1}(\tau^{1},a^{1}), \ldots, \arg\max_{a^n \in \mathcal{A}^{n}} Q^{n}(\tau^{n},a^{n}))$. Two sufficient conditions, the additivity and the monotonicity, to satisfy IGM are proposed in \citet{Sunehag2017} and \citet{rashid2018qmix} respectively. In addition to the $V$ function and $Q$ function decomposition, QPLEX \citep{DBLP:conf/iclr/WangRLYZ21} considers implementing IGM in the dueling structure where $Q=V+A$. QPLEX only constrains the advantage functions to satisfy the IGM principle. The global advantage function is decomposed as $A_{tot}(\bm{\tau}, \bm{a})=\sum_{i=1}^{n}\lambda_{i}(\bm{\tau}, \bm{a})A_{i}(\bm{\tau}, a_{i})$, where $\lambda_{i}(\bm{\tau}, \bm{a}) > 0$. We evaluate the performance of Qmix in \Cref{tab:smac} and \Cref{tab:app_smac}. Integrating the IGM principle into A2PO without compromising the monotonic improvement guarantee is a desirable extension. Specifically, the advantage-based IGM establishes a connection between the global advantage function and the local advantage functions, and the advantage decomposition $A_{tot}(\bm{\tau}, \bm{a})=\sum_{i=1}^{n}\lambda_{i}(\bm{\tau}, \bm{a})A_{i}(\bm{\tau}, a_{i})$ will not jeopardize the derivation of the monotonic improvement guarantee.

\textbf{Convergence and optimality of MARL}. T-PPO \citep{Ye2022} firstly introduce a framework called Generalized Multi-Agent Actor-Critic with Policy Factorization (GPF-MAC), which consists of methods with factorized local policies and may become stuck in sub-optimality. To address this problems, T-PPO transforms a multi-agent MDP into a special "single-agent" MDP with a sequential structure. T-PPO transforms a multi-agent MDP into a "single-agent" MDP with a sequential structure to address this issue. T-PPO has been shown to produce an optimal policy if implemented properly. Theoretically, sequential update methods, such as A2PO and HAPPO, are also instances of GPF-MAC and may be stuck into sub-optimal policies. The main differences between A2PO and T-PPO include that A2PO updates the factorized policies sequentially and makes decisions simultaneously, while T-PPO makes decisions sequentially, and that A2PO does not introduce the virtual state and the sequential transformation framework network. And theoretically, T-PPO may compromise the monotonic improvement guarantee. In \Cref{tab:smac_t_ppo}, we compare A2PO, MAPPO and T-PPO on SMAC tasks empirically. A2PO is superior to T-PPO in the majority of tasks.

\begin{table}[htb]
    \centering
    \small
    \caption{Comparisons of A2PO, MAPPO and T-PPO.}
    \begin{tabular}{l|l|l|l|l}
        \toprule
        Map & Difficulty & MAPPO w/ PS & T-PPO w/ PS & A2PO w/ PS \\
        \midrule
        1c3s5z & Easy & $\bm{100(0.0)}$ & $99.8((0.0)$ & $\bm{100(0.0)}$ \\
        MMM2 & Super Hard & $90.6(8.9)$ & $81.6( 7.7)$ & $\bm{98.4(1.3)}$ \\
        3s5z\_vs\_3s6z & Super Hard & $82.8(19.2)$ & $85.5(5.2)$ & $\bm{93.8(19.8)}$ \\
        6h\_vs\_8z & Super Hard & $87.5(1.5)$ & $\bm{91.8(1.1)}$ & $90.6(1.3)$ \\
        corridor & Super Hard & $99.1(0.3)$ & $96.9(0.0)$ & $\bm{100(0.0)}$ \\
        \bottomrule
    \end{tabular}
    \label{tab:smac_t_ppo}
\end{table}

\section{The Related Work of Coordinate Descent}\label{app:rw}

Realizing the similarity between the sequential policy update scheme and the block coordinate descent algorithms, we borrow the optimization techniques in the coordinate descent algorithms to accelerate the optimization and amplify the convergence advantage over the simultaneous update scheme \citep{Gordon2015,Shi2017}. One of the critical questions in the coordinate descent algorithms is selecting the coordinate for the next-step optimization. \citet{Glasmachers2013,Lu2018} provided analyses of the convergence rate advantage of the Gauss-Southwell rule, i.e., greedily selecting the coordinate with the maximal gradient, over the random selection rule. 
% \citet{Lu2018} verified that alternating in the Gauss-Southwell rule and the random rule has a faster convergence rate than other common selection rules. 
We recognize the optimization of our surrogate objective \citep{Schulman2017} agent-by-agent as a block coordinate descent problem. Therefore the agent selection rule plays a crucial role in accelerating the optimization. Inspired by the coordinate selection rules, we propose greedy and semi-greedy agent selection rules and empirically show that the underperforming agents benefit from the greedily selecting agents.